\documentclass{article}

\usepackage[final]{neurips_2021}

\usepackage[utf8]{inputenc} 
\usepackage[T1]{fontenc}    
\usepackage{hyperref}       
\usepackage{url}            
\usepackage{booktabs, colortbl, array} 
\usepackage{amsfonts, amsmath, amssymb, amsthm}       
\usepackage{nicefrac}       
\usepackage{microtype}      
\usepackage{xcolor}         
\usepackage{graphicx, subcaption}
\usepackage{algorithm, algorithmic}
\definecolor{Gray}{gray}{0.9} 

\newcommand{\WYK}{\textcolor{blue}}

\makeatletter
\newtheorem{thm}{\protect\theoremname}
\theoremstyle{remark}
\newtheorem*{rem}{\protect\remarkname}
\newtheorem{lem}[thm]{\protect\lemmaname}
\newtheorem{prop}[thm]{\protect\propositionname}
\makeatother

\usepackage{babel}
\providecommand{\lemmaname}{Lemma}
\providecommand{\propositionname}{Proposition}
\providecommand{\remarkname}{Remark}
\providecommand{\theoremname}{Theorem}

\title{Doubly Robust Thompson Sampling \\ with Linear Payoffs}

\author{
  Wonyoung Kim \\
  Department of Statistics\\
  Seoul National University\\
  \texttt{eraser347@snu.ac.kr} \\
  \And
  Gi-Soo Kim \\
  Department of Industrial Engineering \&\\
  Artificial Intelligence Graduate School\\
  UNIST\\
  \texttt{gisookim@unist.ac.kr} \\
  \And
  Myunghee Cho Paik \\
  Department of Statistics\\
  Seoul National University\\
  Shepherd23 Inc.\\
  \texttt{myungheechopaik@snu.ac.kr} \\
}

\begin{document}

\global\long\def\Expectation{\mathbb{E}}%
\global\long\def\Probability{\mathbb{P}}%
\global\long\def\CE#1#2{\Expectation\left[\left.#1\right|#2\right]}%
\global\long\def\CP#1#2{\Probability\left(\left.#1\right|#2\right)}%
\global\long\def\Real{\mathbb{R}}%
\global\long\def\abs#1{\left|#1\right|}%
\global\long\def\norm#1{\left\Vert #1\right\Vert }%
\global\long\def\Regret#1{regret\ensuremath{(#1)}}%
\global\long\def\Psuedoregret#1{regret\prime\ensuremath{(#1)}}%
\global\long\def\Optimalarm#1{a^{*}_{#1}}%
\global\long\def\Action#1{a_{#1}}%
\global\long\def\Filtration#1{\mathcal{F}_{#1}}%
\global\long\def\History#1{\mathcal{H}_{#1}}%
\global\long\def\Context#1#2{X_{#1}(#2)}%
\global\long\def\Reward#1#2{Y_{#1}(#2)}%
\global\long\def\SelectionP#1#2{\pi_{#1}(#2)}%
\global\long\def\Pisampled#1#2{\tilde{\pi}_{#1}(#2)}%
\global\long\def\DRreward#1#2{Y_{#1}^{DR}(#2)}%
\global\long\def\BetaSampled#1#2{\tilde{\beta}_{#1}(#2)}%
\global\long\def\Betatrunc#1{\breve{\beta}_{#1}}%
\global\long\def\Estimator#1{\widehat{\beta}_{#1}}%
\global\long\def\Error#1#2{\eta_{#1}(#2)}%
\global\long\def\DRError#1#2{\widehat{\eta}_{#1}(#2)}%
\global\long\def\Indicator#1{\mathbb{I}\left(#1\right)}%
\global\long\def\Diff#1#2{\Delta_{#1}(#2)}%
\global\long\def\Mineigen#1{\lambda_{\min}\left(#1\right)}%
\global\long\def\Sampledreward#1#2{\tilde{Y}_{#1}(#2)}%
\global\long\def\Gammaset#1{\tilde{\Gamma}_{#1}}%
\global\long\def\Sampledaction#1#2{a^{(#2)}_{#1}}%

\maketitle

\begin{abstract}
A challenging aspect of the bandit problem is that a stochastic reward is observed only for the chosen arm and the rewards of other arms remain missing.    
The dependence of the arm choice on the past context and reward pairs compounds the complexity of regret analysis.
We propose a novel multi-armed contextual bandit algorithm called Doubly Robust (DR) Thompson Sampling employing the doubly-robust estimator used in missing data literature to Thompson Sampling with contexts (\texttt{LinTS}).
Different from previous works relying on missing data techniques (\citet{dimakopoulou2019balanced}, \citet{kim2019doubly}), the proposed algorithm is designed to allow a novel additive regret decomposition leading to an improved regret bound with the order of $\tilde{O}(\phi^{-2}\sqrt{T})$, where $\phi^2$ is the minimum eigenvalue of the covariance matrix of contexts.
This is the first regret bound of \texttt{LinTS} using $\phi^2$ without the dimension of the context, $d$.
Applying the relationship between $\phi^2$ and $d$, the regret bound of the proposed algorithm is $\tilde{O}(d\sqrt{T})$ in many practical scenarios, improving the bound of \texttt{LinTS} by a factor of $\sqrt{d}$.
A benefit of the proposed method is that it utilizes all the context data, chosen or not chosen, thus allowing to circumvent the technical definition of unsaturated arms used in theoretical analysis of \texttt{LinTS}.
Empirical studies show the advantage of the proposed algorithm over \texttt{LinTS}.
\end{abstract}


\section{Introduction}
Contextual bandit has been popular in sequential decision tasks such as news article recommendation systems.  
In bandit problems, the learner sequentially pulls one arm among multiple arms and receives random rewards on each round of time.
While not knowing the compensation mechanisms of rewards, the learner should make  his/her decision to maximize the cumulative sum of rewards.  
In the course of gaining information about the compensation mechanisms through feedback, the learner should carefully balance between exploitation, pulling the best arm based on information accumulated so far, and exploration, pulling the arm that will assist in future choices, although it does not seem to be the best option at the moment.  
Therefore in the bandit problem, estimation or learning is an important element besides decision making.

A challenging aspect of estimation in the bandit problem is that a stochastic reward is observed only for the chosen arm. 
Consequently, only the context and reward pair of the chosen arm is used for estimation, which causes dependency of the context data at the round on the past contexts and rewards.
To handle this difficulty, we view bandit problems as missing data problems.  
The first step in handling missing data is to define full, observed, and missing data. 
In bandit settings, full data consist of rewards and contexts of all arms; observed data consist of full contexts for all arms and the reward for the chosen arm; missing data consist of the rewards for the arms that are not chosen.   
Typical estimation procedures require both rewards and contexts pairs to be observed, and the observed contexts from the unselected are discarded (see Table \ref{tab:missingdata_bandit}).
The analysis based on the completely observed pairs only  is called {\it complete record analysis}. 
Most stochastic bandit algorithms utilize estimates based on {\it complete record analysis}.  
Estimators from {\it complete record analysis} are known to be inefficient. 
In bandit setting, using the observed data whose probability of observation depends on previous rewards requires special theoretical treatment.

\begin{table}
  \caption{The shaded data are used in \textit{complete record analysis} (left) and DR method (right) under multi-armed contextual bandit settings. 
  The contexts, rewards and DR imputing values are denoted by \(X\), \(Y\), and \(Y^{DR}\), respectively.
  The question mark refers to the missing reward of unchosen arms.}
  \label{tab:missingdata_bandit}
  \centering
  \def\arraystretch{1.2}
  \footnotesize
  \begin{tabular}{ccccc}
  \toprule
   & \multicolumn{2}{c}{\( t=1 \)}   & \multicolumn{2}{c}{\( t=2 \)}  \\ \hline
  Arm 1 & \(\Context{1}{1}\)  & ?  & \(\Context{1}{2}\)  & ?   \\ \hline
  Arm 2 & \(\Context{2}{1}\) & ?   & \cellcolor{Gray}{\(X_{\Action{2}}(2)\)} & \cellcolor{Gray}{\(Y_{\Action{2}}(2)\)} \\ \hline
  Arm 3 & \cellcolor{Gray}{\(X_{\Action{1}}(1)\)} & \cellcolor{Gray}{\(Y_{\Action{1}}(1)\)}  & \(\Context{3}{2} \) & ? \\ \hline
  Arm 4 & \(\Context{4}{1} \) & ? & \(\Context{4}{2}\) & ?  \\ \bottomrule
  \end{tabular}
  \hskip 3pt
  \footnotesize
  \begin{tabular}{ccccc}
  \toprule
  & \multicolumn{2}{c}{\( t=1 \)}   & \multicolumn{2}{c}{\( t=2 \)}  \\ \hline
  Arm 1 &\cellcolor{Gray}{\(\Context{1}{1} \)}& \cellcolor{Gray}{\(\DRreward{1}{1}\)}  & \cellcolor{Gray}{\(\Context{1}{2}\)}  & \cellcolor{Gray}{\(\DRreward{1}{2}\)}  \\ \hline
  Arm 2 & \cellcolor{Gray}{\(\Context{2}{1} \)} & \cellcolor{Gray}{\(\DRreward{2}{1}\)}   & \cellcolor{Gray}{\(X_{\Action{2}}(2)\)} & \cellcolor{Gray}{\(Y^{DR}_{\Action{2}}(2) \)} \\ \hline
  Arm 3 & \cellcolor{Gray}{\(X_{\Action{1}}(1)\)} & \cellcolor{Gray}{\(Y^{DR}_{\Action{1}}(1)\)}  & \cellcolor{Gray}{\(\Context{3}{2} \)} & \cellcolor{Gray}{\(\DRreward{3}{2}\)}  \\ \hline
  Arm 4 & \cellcolor{Gray}{\(\Context{4}{1} \)} & \cellcolor{Gray}{\(\DRreward{4}{1}\)} & \cellcolor{Gray}{\(\Context{4}{2}\)} & \cellcolor{Gray}{\(\DRreward{4}{2}\)}  \\ 
  \bottomrule
  \end{tabular}
\vskip -5pt
\end{table}

There are two main approaches to missing data: imputation and inverse probability weighting (IPW).
Imputation is to fill in the predicted value of missing data from a specified model, and IPW is to use the observed records only but weight them by the inverse of the observation probability. 
The doubly robust (DR) method \citep{robins1994, bang2005doubly} is a combination of imputation and IPW tools. 
We provide a review of missing data and DR methods in supplementary materials.  The robustness against model misspecification in missing data settings is insignificant in the bandit setting since the probability of observation or allocation to an arm is known.  The merit of the DR method in the bandit setting is its ability to employ all the contexts including unselected arms. 

We propose a novel multi-armed contextual bandit algorithm called Doubly Robust Thompson Sampling (\texttt{DRTS}) that applies the DR technique used in missing data literature to Thompson Sampling with linear contextual bandits (\texttt{LinTS}).  
The main thrust of \texttt{DRTS} is to utilize contexts information for all arms, not just chosen arms.
By using the unselected, yet observed contexts, along with a novel algorithmic device, the proposed algorithm renders a unique regret decomposition which leads to a novel regret bound without resorting to the technical definition of unsaturated arms used by \citet{agrawal2014thompson}.
Since categorizing the arms into saturated vs. unsaturated plays a critical role in costing extra $\sqrt{d}$, by circumventing it, we prove a $\tilde{O}(d\sqrt{T})$ bound of the cumulative regret in many practical occasions compared to $\tilde{O}(d^{3/2}\sqrt{T})$ shown in \citet{agrawal2014thompson}.

The main contributions of this paper are as follows.
\vspace{-5pt}
\begin{itemize}
    \item We propose a novel contextual bandit algorithm that improves the cumulative regret bound of \texttt{LinTS} by a factor of $\sqrt{d}$ (Theorem \ref{thm:regret_bound}) in many practical scenarios (Section \ref{subsec:regret_bound}).  
    This improvement is attained mainly by defining a novel set called {\it super-unsaturated} arms, that is utilizable due to the proposed estimator and resampling technique adopted in the algorithm.
    \item We provide a novel estimation error bound of the proposed estimator (Theorem \ref{thm:Estimation_error}) which depends on the minimum eigenvalue of the covariance matrix of the contexts from all arms without $d$.
    
    \item We develop a novel dimension-free concentration inequality for sub-Gaussian vector martingale (Lemma \ref{lem:eta_x_bound}) and use it in  deriving our regret bound in place of the self-normalized theorem by \citet{abbasi2011improved}.
    
    \item We develop a novel concentration inequality for the bounded matrix martingale (Lemma \ref{lem:min_eigenvalue_concentration}) which improves the existing result (Proposition \ref{prop:tropp_matrix_concentration}) by removing the dependency on $d$ in the bound. 
     Lemma \ref{lem:min_eigenvalue_concentration} also allows eliminating the forced sampling phases required in some bandit algorithms relying on Proposition \ref{prop:tropp_matrix_concentration}    \citep{amani2019linear,bastani2020online}.
\end{itemize}
All missing proofs are in supplementary materials.

\section{Related works}
Thompson Sampling \citep{1933Thompson} has been extensively studied and shown solid performances
in many applications 
(e.g. \cite{chapelle2011}).
\citet{agrawal2013thompson} is the first to prove theoretical bounds for \texttt{LinTS} and an alternative proof is given by \citet{abeille2017linear}.
Both papers show $\tilde{O}(d^{3/2}\sqrt{T})$ regret bound, which is known as the best regret bound for \texttt{LinTS}.  
Recently, \citet{hamidi2020worstcase} points out that $\tilde{O}(d^{3/2}\sqrt{T})$ could be the best possible one can get when the estimator used by \texttt{LinTS} is employed.
In our work, we improve this regret bound by a factor of $\sqrt{d}$ in many practical scenarios through a novel definition of super-unsaturated arms, which becomes utilizable due to the proposed estimator and resampling device implemented in the algorithm.

Our work assumes the independence of the contexts from all arms across time rounds.
Some notable works have used the assumption that the contexts are independently identically distributed (IID). 
Leveraging the IID assumption with a margin condition, \citet{goldenshluger2013linear} derives a two-armed linear contextual bandit algorithm with a regret upper bound of order \(O(d^3\mathrm{log}T)\).
\citet{bastani2020online} has extended this algorithm to any number of arms and improves the regret bound to \(O(d^2\mathrm{log}^{\frac{3}{2}}d\cdot\mathrm{log}T)\).
The margin condition states that the gap between the expected rewards of the optimal arm and the next best arm is nonzero with some constant probability.
This condition is crucial in achieving a \(O(\mathrm{log}T)\) regret bound instead of  \(\tilde{O}(\sqrt{T})\). 
In this paper, we do not assume this margin condition, and focus on the dependence on the dimension of contexts \(d\).


From a missing data point of view, most stochastic contextual bandit algorithms use the estimator from \textit{complete record analysis} except 
\citet{dimakopoulou2019balanced} and \citet{kim2019doubly}.
\citet{dimakopoulou2019balanced} employs an IPW estimator that is based on the selected contexts alone.
\citet{dimakopoulou2019balanced} proves a $\tilde{O}(d\sqrt{\epsilon^{-1}T^{1+\epsilon}N})$ regret bound for 
their algorithm
which depends on the number of arms, $N$. 
\citet{kim2019doubly} considers the high-dimensional settings with sparsity, utilizes a DR technique, and improves the regret bound in terms of the sparse dimension instead of the actual dimension of the context, $d$.  
 \citet{kim2019doubly} is different from ours in several aspects: the mode of exploration   ($\epsilon$-greedy vs. Thompson Sampling), the mode of regularization (Lasso vs. ridge regression); and the form of the estimator.  
A sharp distinction between the two estimators lies in that \citet{kim2019doubly} aggregates contexts and rewards over the arms although they employ all the contexts. 
If we apply this aggregating estimator and DR-Lasso bandit algorithm to the low-dimensional setting, we obtain a regret bound of order $O(\frac{Nd}{\phi^2}\sqrt{T})$ when the contexts from the arms are independent. 
This bound is bigger than our bound by a factor of $d$ and $N$. 
It is because the aggregated form of the estimator does not permit the novel regret decomposition derived in Section \ref{subsec:regret_analysis}. 
The proposed estimator coupled with a novel algorithmic device renders the additive regret decomposition which in turn improves the order of the regret bound.


\section{Proposed estimator and algorithm}
\subsection{Settings and assumptions}
\label{subsec:Settings_and_assumptions}
We denote a $d$-dimensional context for the $i^{th}$ arm at round $t$ by $\Context it\in \Real^{d}$, and the corresponding random reward by $\Reward it$ for $i=1,\ldots,N$.
We assume $\CE{\Reward it}{\Context it}=\Context it^{T}\beta$ for some unknown parameter $\beta\in\mathbb{R}^d$.
At round $t$,  the arm that the learner chooses is denoted by $\Action{t}\in\{1,\ldots,N\}$, and the optimal arm by $\Optimalarm t:=\arg\max_{i=1,\ldots,N}\left\{ \Context it^{T}\beta\right\}$.
Let $\Regret t$ be the difference between the expected reward of the chosen arm and the optimal arm at round $t$, i.e.,  $\Regret t:=\Context{\Optimalarm t}t^{T}\beta-\Context{\Action{t}}t^{T}\beta$.
The goal is to minimize the sum of regrets over $T$ rounds, $R(T):=\sum_{t=1}^{T}\Regret t$.
The total round $T$ is finite but possibly unknown. 
We also make the following assumptions.

\textbf{Assumption 1. Boundedness for scale-free regrets.} 
For all $i=1,\ldots,N$ and $t=1, \ldots, T$, we have $\norm{\Context it}_{2}\le1$
and $\norm{\beta}_{2}\le1$.

\textbf{Assumption 2. Sub-Gaussian error.} Let
$
\History{t} := \bigcup_{\tau=1}^{t-1} \left[ \{\Context{i}{\tau}\}_{i=1}^{N}  \cup  \{\Action{\tau}\} \cup \{\Reward{\Action{\tau}}{\tau}\} \right] \cup \{\Context{i}{t}\}_{i=1}^{N}
$
be the set of observed data at round $t$. 
For each $t$ and $i$, the error $\Error it:=\Reward it-\Context it^{T}\beta$ is conditionally zero-mean $\sigma$-sub-Gaussian for a fixed constant $\sigma\ge0$, i.e, $\CE{\Error it}{\History t}=0$ and $\CE{\exp\left(\lambda\Error it\right)}{\History t}\le\exp(\lambda^{2}\sigma^{2}/2)$, for all $\lambda\in\Real$. 
Furthermore, the distribution of $\Error it$ does not depend on the choice at round $t$, i.e. $\Action{t}$.

\textbf{Assumption 3. Independently distributed contexts.} 
The stacked contexts vectors $\{\Context i1\}_{i=1}^{N},\ldots,\{\Context{i}{T}\}_{i=1}^{N}\in \Real^{d N}$ are independently distributed.

\textbf{Assumption 4. Positive minimum eigenvalue of the average of covariance matrices.} 
For each $t$, there exists a constant $\phi^{2}>0$ such that
$\Mineigen{\Expectation\left[\frac{1}{N}\sum_{i=1}^{N}\Context it\Context it^{T}\right]}\ge\phi^{2}.$

Assumptions 1 and 2 are standard in stochastic bandit literature \cite{agrawal2013thompson}.
We point out that given round $t$, Assumption 3 allows that the contexts among different arms, $\Context 1t,\ldots,\Context Nt$ are correlated to each other.  
Assumption 3 is weaker than the assumption of IID, and the IID condition is considered by \citet{goldenshluger2013linear} and \citet{bastani2020online}.
As \citet{bastani2020online} points out, the IID assumption is reasonable in some practical settings, including clinical trials, where health outcomes of patients are independent of those of other patients.
Both \citet{goldenshluger2013linear} and \citet{bastani2020online} address the problem where the contexts are equal across all arms, i.e. $X(t)=\Context{1}{t}=\ldots=\Context{N}{t}$, while our work admits different contexts over all arms.
Assumption 4 guarantees that the average of covariance matrices of contexts over the arms is well-behaved so that the inverse of the sample covariance matrix is bounded by the spectral norm. 
This assumption helps controlling the estimation error of $\beta$ in linear regression models. 
Similar assumptions are adopted in existing works in the bandit setting \citep{goldenshluger2013linear,amani2019linear,li2017provably,bastani2020online}.

\subsection{Doubly robust estimator}
To describe the contextual bandit DR estimator, let $\pi_{i}(t):=\CP{\Action{t}=i}{\History t} > 0$ be the probability of selecting arm $i$ at round $t$. 
We define a DR pseudo-reward as 
\begin{equation}
\DRreward it= \left\{ 1-\frac{\Indicator{i=\Action{t}}}{\SelectionP it}\right\} \Context it^{T}\Betatrunc{t}
+\frac{\Indicator{i=\Action{t}}}{\SelectionP it}\Reward{\Action{t}}t,
\label{eq:DRreward}
\end{equation}
for some $\Betatrunc{t}$ depending on $\History t$. 
Background of missing data methods and derivation of the DR pseudo-reward is provided in the supplementary material. 
Now, we propose our new estimator $\Estimator t$ 
with a regularization parameter $\lambda_{t}$ as below: 
\begin{equation}
\Estimator t = \left(\sum_{\tau=1}^{t} \sum_{i=1}^{N}\Context i{\tau}\Context i{\tau}^{T}+\lambda_{t}I\right)^{-1}\left(\sum_{\tau=1}^{t} \sum_{i=1}^{N}\Context i{\tau}\DRreward i{\tau}\right).
\label{eq:estimator}
\end{equation}
Harnessing the pseudo-rewards defined in (\ref{eq:DRreward}), we can make use of all contexts rather than just selected contexts. 
The DR estimator by \citet{kim2019doubly} utilizes all contexts but has a different form from ours.
While \citet{kim2019doubly} uses Lasso estimator with pseudo-rewards {\it aggregated} over all arms, we use ridge regression estimator with pseudo-rewards in \eqref{eq:DRreward} which are defined {\it separately} for each $i=1,\ldots, N$.
This seemingly small but important difference in forms paves a way in rendering our unique regret decomposition and improving the regret bound. 

\subsection{Algorithm}
\label{subsec:DRTS}
In this subsection, we describe our proposed algorithm, \texttt{DRTS} which adapts DR technique to \texttt{LinTS}.
The \texttt{DRTS} is presented in Algorithm \ref{alg:DR_Thompson_Sampling}. 
Distinctive features of \texttt{DRTS} compared to \texttt{LinTS} include the novel estimator and the resampling technique. 
At each round $t\ge1$, the algorithm samples $\BetaSampled it$ from the distribution $N(\Estimator{t-1},v^{2}V_{t-1}^{-1})$ for each $i$ independently.  
Let $\Sampledreward{i}{t}:=\Context{i}{t}^T\BetaSampled{i}{t}$ and $m_{t}:=\arg\max_{i} \Sampledreward{i}{t}$.
We set $m_t$ as a candidate action and compute $\tilde{\pi}_{m_t}(t):=\Probability(\Sampledreward{m_t}t=\max_{i}\Sampledreward it|\History t)$.
\footnote{This computation is known to be challenging but employing the independence among $\BetaSampled{1}{t},\ldots,\BetaSampled{N}{t}$, we derive an explicit form approximating $\tilde{\pi}_{m_t}(t)$ in supplementary materials Section H.1.}
If $\tilde{\pi}_{m_t}(t)>\gamma$, then the arm $m_t$ is selected, i.e., $a_t=m_t$. 
Otherwise, the algorithm resamples $\BetaSampled it$ until it finds another arm satisfying $\Pisampled{i}{t}>\gamma $ up to a predetermined fixed value $M_{t}$.
Section \ref{subsec:maximum_possible_resampling} in supplementary materials describes issues related to $M_t$ including a suitable choice of $M_t$.

\begin{algorithm}[tb]
    \caption{Doubly Robust Thompson Sampling for Linear Contextual Bandits (\texttt{DRTS})}
    \label{alg:DR_Thompson_Sampling}    
\begin{algorithmic} 
    \STATE {\bfseries Input:} Exploration parameter $v>0$, Regularization parameter $\lambda>0$, Selection probability threshold $\gamma\in[1/(N+1),1/N)$, Imputation estimator $\Betatrunc u=f(\{X(\tau),\Reward{\Action{\tau}}{\tau}\}_{\tau=1}^{u-1})$, Number of maximum possible resampling $M_{t}$.
    \STATE Set $F_{0}=0$, $W_{0}=0$, $\Estimator{0}=0$ and $V_{0}=\lambda I$
    \FOR{$t=1$ {\bfseries to} $T$}
    \STATE Observe contexts $\{\Context it\}_{i=1}^{N}$.
    \STATE Sample $\BetaSampled{1}{t},\ldots,\BetaSampled{N}{t}$ from $N(\Estimator{t-1},v^{2}V_{t-1}^{-1})$ independently. Compute $\Sampledreward{i}{t}=\Context{i}{t}^T\BetaSampled{i}{t}$
    \STATE Observe a candidate action $m_{t}:=\arg\max_{i}\Sampledreward{i}{t}$.
    \STATE Compute $\Pisampled{m_{t}}t:=\CP{\max_{i}\Sampledreward{i}{t} = \Sampledreward{m_t}{t}}{\History{t}}$.
    \FOR{$l=1$ {\bfseries to} $M_{t}$}
        \IF{$\Pisampled{m_t}{t} \le \gamma$}
        \STATE Sample another $\BetaSampled{1}{t},\ldots,\BetaSampled{N}{t}$, observe another $m_t$, and update $\Pisampled{m_{t}}t$.
        \ELSE
        \STATE Break.
        \ENDIF
    \ENDFOR
    \STATE Set $\Action{t}=m_t$, and play arm $\Action{t}$.
    \STATE Observe reward $\Reward{\Action{t}}t$ and compute $\DRreward{i}{t}$
    \STATE $F_{t}=F_{t-1}+\sum_{i=1}^{N}\Context{i}{t}\DRreward{i}{t}$; 
    $W_{t}=W_{t-1}+\sum_{i=1}^{N}\Context{i}{t}\Context{i}{t}^{T}$;
    $V_{t}=W_{t}+\lambda\sqrt{t}I$
    \STATE $\Estimator{t}=V_t^{-1}F_{t}$
    \STATE Update $\Betatrunc{t+1}$ for next round.
    \ENDFOR
    \vskip -10pt
\end{algorithmic}
\end{algorithm}

The resampling step is incorporated to avoid small values of the probability of selection so that the pseudo-reward in (\ref{eq:DRreward}) is numerically stable. 
A naive remedy to stabilize the pseudo-reward is to use $\max\{\SelectionP{i}{t},\gamma\}$, which fails to leading to our regret bound since it induces bias and also cannot guarantee that the selected arm is in the super-unsaturated arms defined in (\ref{eq:super_unsaturated_set}) with high probability (For details, see Section \ref{subsec:regret_analysis}).
The resampling step implemented in the proposed algorithm is designed to solve these problems. 

\section{Theoretical results}
Our theoretical results are organized as follows.  
In Section \ref{subsec:regret_bound}, we provide the main result, the cumulative regret bound of 
$\tilde{O}(\phi^{-2}\sqrt{T})$ of \texttt{DRTS}.
The main thrust of deriving the regret bound is to define super-unsaturated arms.   
In Section \ref{subsec:regret_analysis} we introduce the definition of super-unsaturated arms and show how it admits a novel decomposition of the regret into two additive terms as in (\ref{eq:decomposition}).
In Section \ref{subsec:dim_free_DR_estimator} we bound each term of the decomposed regret bounds (\ref{eq:decomposition}).  
The first term is the estimation error, and Theorem \ref{thm:Estimation_error} finds its bound.  
In the course of proving Theorem \ref{thm:Estimation_error}, we need Lemma \ref{lem:eta_x_bound}, which plays a similar role to the self-normalized theorem of \citet{abbasi2011improved}.
We conclude the section by presenting Lemma \ref{lem:min_eigenvalue_concentration}  and bound the second term of (\ref{eq:decomposition}).

\subsection{An improved regret bound}
\label{subsec:regret_bound}
Theorem \ref{thm:regret_bound} provides the regret bound of \texttt{DRTS} in terms of the minimum eigenvalue without $d$.

\begin{thm}
Suppose that Assumptions 1-4 hold.
If $\Betatrunc t$ in Algorithm \ref{alg:DR_Thompson_Sampling} satisfies $\|\Betatrunc t-\beta\|_{2}\le b$ for a constant $b>0$, ~for all $t=1,\ldots,T$, then with probability $1-2\delta$, the cumulative regret by time $T$ for \texttt{DRTS} algorithm is bounded by
\begin{equation}
R(T)\le 2 + \frac{4C_{b,\sigma}}{\phi^{2}}\sqrt{T\log\frac{12T^{2}}{\delta}}+\frac{2\sqrt{2T}}{\phi\sqrt{N}},
\label{eq:regret_bound}
\end{equation}
where $C_{b,\sigma}$ is a constant which depends only on $b$ and $\sigma$.
\label{thm:regret_bound}
\end{thm}

The bound (\ref{eq:regret_bound}) has a rate of $O(\phi^{-2}\sqrt{T})$.
The relationship between the dimension $d$ and the minimum eigenvalue $\phi^{2}$ can be shown by
$$
d\phi^{2}=\frac{d}{N}\Mineigen{\Expectation\sum_{i=1}^{N}\Context it\Context  it^{T}} \\
\le \frac{1}{N}\Expectation\sum_{i=1}^{N}\text{Tr}\left(\Context it\Context it^{T}\right) 
=\frac{1}{N}\Expectation\sum_{i=1}^{N}\norm{\Context it}_{2}^{2}\le1.
$$
This implies $\phi^{-2}\ge d$,
\footnote{
Some previous works assume $\phi^{-2}=O(1)$ even when $\|\Context{i}{t}\|_2 \le 1$ (e.g. \citet{li2017provably}). 
As pointed out by \citet{ding2021efficient}, this assumption is unrealistic and the reported regret bound should be multiplied by $O(d)$.
}
but there are many practical scenarios such that $\phi^{-2}=O(d)$ holds.
\citet{bastani2021mostly} identifies such examples including the uniform distribution and truncated multivariate normal distributions.
When the context has uniform distribution on the unit ball, $\phi^{-2}=d+2$.  
When the context has truncated multivariate normal distribution with mean 0 and covariance $\Sigma$, we can set $\phi^{-2} = (d+2)\exp(\frac{1}{2\lambda_{\min}(\Sigma)})$.
For more examples, we refer to \citet{bastani2021mostly}.
Furthermore, regardless of distributions, $\phi^{-2}=O(d)$ holds when the correlation structure has the row sum of off-diagonals independent of the dimension, for example, AR(1), tri-diagonal, block-diagonal matrices.
In these scenarios, the regret bound in \eqref{eq:regret_bound} becomes $\tilde{O}(d\sqrt{T})$.
Compared to the previous bound of \texttt{LinTS} \citep{agrawal2014thompson, abeille2017linear}, we obtain a better regret bound by the factor of $\sqrt{d}$ for identified practical cases.

As for the imputation estimator $\check{\beta}_t$, we assume that $\|\check{\beta}_t-\beta\|_2 \le b$, where $b$ is an absolute constant. 
We suggest two cases which guarantee this assumption.
First, if a biased estimator is used, we can rescale the estimator so that its $l_2$-norm is bounded by some constant $C>0$.
Then, $\|\check{\beta}_t-\beta\|_2 \le \|\check{\beta}_t\|_2 +\|\beta\|_2 \le C+1$ and $b = C+1$.
Second, consistent estimators such as ridge estimator or the least squared estimator satisfy the condition since $\|\check{\beta}_t-\beta\|_2=O(d \sqrt{\log t /t})$.
The term $d$ is cancelled out when $t \ge t_d$, where $t_d$ is the minimum integer that satisfies $\log t / t \le d^{-2}$.
In these two cases, we can find a constant $b$ which satisfies the assumption on the imputation estimator $\check{\beta}_t$.

\subsection{Super-unsaturated arms and a novel regret decomposition}
\label{subsec:regret_analysis}
The key element in deriving (\ref{eq:regret_bound}) is to decompose the regret into two additive terms as in (\ref{eq:decomposition}).  
To allow such decomposition to be utilizable, we need to define a novel set of arms called super-unsaturated arms, which replaces the role of unsaturated arms in \citep{agrawal2014thompson}.  
The super-unsaturated arms are formulated so that the chosen arm is included in this set with high probability.
For each $i$ and $t$, let $\Diff it := \Context{\Optimalarm t}{t}^T\beta - \Context{i}{t}^T\beta$. 
Define $A_{t}:=\sum_{\tau=1}^{t}\Context{\Action{\tau}}{\tau}\Context{\Action{\tau}}{\tau}^{T}+\lambda I$ and $V_{t}:=\sum_{\tau=1}^{t}\sum_{i=1}^{N} \Context{i}{\tau}\Context{i}{\tau}^{T}+\lambda_{t} I$.
For the sake of contrast, recall the definition of unsaturated arms by \citet{agrawal2014thompson} 
\begin{equation}
U_t := \left\{ i:\Diff it \le g_t\norm{\Context it}_{A_{t-1}^{-1}}\right\} ,
\label{eq:unsaturated_arms}
\end{equation}
where $g_t:=C\sqrt{d\log(t/\delta)} \min \{\sqrt{d}, \sqrt{\log N}\}$ for some constant $C>0$.
This $g_t$ is constructed to ensure that there exists a positive lower bound for the probability that the selected arm is unsaturated.
In place of (\ref{eq:unsaturated_arms}), we define a set of super-unsaturated arms for each round \(t\) by
\begin{equation}
N_t :=\biggl\{  i: \Diff it \le2\norm{\Estimator{t-1}-\beta}_{2}+\sqrt{\norm{\Context{\Optimalarm{t}}{t}}_{V^{-1}_{t-1}}^2+\norm{\Context{i}{t}}_{V^{-1}_{t-1}}^2}\biggr\}.
\label{eq:super_unsaturated_set}
\end{equation}

While $g_t\norm{\Context it}_{A_{t-1}^{-1}}$ in (\ref{eq:unsaturated_arms}) is normalized with only selected contexts,
the second term in the right hand side of (\ref{eq:super_unsaturated_set}) is normalized with all contexts including $\Context{\Optimalarm{t}}{t}$, the contexts of the optimal arm.
This bound of $\Diff{i}{t}$ plays a crucial role in bounding the regret with a novel decomposition as in \eqref{eq:decomposition}.
The following Lemma shows a lower bound of the probability that the candidate arm is super-unsaturated.

\begin{lem}
For each $t$, let $m_{t}:=\arg\max_{i} \Sampledreward{i}{t}$ and let $N_t$ be the super-unsaturated arms defined in (\ref{eq:super_unsaturated_set}). 
For any given $\gamma\in[1/(N+1),1/N)$, set $v=(2\log\left(N/(1-\gamma N)\right))^{-1/2}$. 
Then, $\CP{m_{t}\in N_t}{\History{t}}\ge 1 - \gamma$.
\label{lem:super_unsaturated_arms} 
\end{lem}

Lemma \ref{lem:super_unsaturated_arms} directly contributes to the reduction of $\sqrt{d}$ in the hyperparameter $v$.
In \citet{agrawal2014thompson}, to prove a lower bound of $\CP{\Action{t}\in U_t }{\History{t}}$, it is required to set $v=\sqrt{9d\log(t/\delta)}$, with the order of $\sqrt{d}$.
In contrast, Lemma \ref{lem:super_unsaturated_arms} shows that $v$ does not need to depend on $d$ due to the definition of super-unsaturated arms in (\ref{eq:super_unsaturated_set}).
In this way, we obtain a lower bound of $\CP{m_t\in N_t }{\History{t}}$ without costing extra $\sqrt{d}$. 

Using the lower bound, we can show that the resampling scheme allows the algorithm to choose the super-unsaturated arms with high probability.
For all $i \notin N_t$,
$$
\Pisampled{i}{t} := \CP{m_{t}=i}{\History{t}} \le \CP{\cup_{j\notin N_{t}} \{m_{t}=j\} }{\History{t}} = \CP{m_{t} \notin N_t}{\History{t}} \le \gamma,
$$
where the last inequality holds due to Lemma \ref{lem:super_unsaturated_arms}.
Thus, in turn, if $\Pisampled{i}{t} > \gamma$,  then $i \in N_t$.
This means that $\{i:\tilde{\pi}_i(t)>\gamma\}$ is a subset of $N_t$ and  
\begin{equation*}
\{\Action{t} \in \{i:\Pisampled{i}{t} >\gamma\} \} \subset \{\Action{t} \in N_{t}\}.
\end{equation*} 
Hence, the probability of the event $\{a_t\in N_t\}$ is greater than the probability of sampling any arm which satisfies $\tilde{\pi}_i(t)>\gamma$. 
Therefore, with resampling, the event $\{\Action{t} \in N_{t}\}$ occurs with high probability. (See supplementary materials Section A for details.)

When the algorithm chooses the arm from the super-unsaturated set, i.e., when $\Action{t}\in N_t$ happens, \eqref{eq:super_unsaturated_set} implies
\begin{equation}
\Diff{\Action{t}}{t} \le 2\norm{\Estimator{t-1}-\beta}_{2} + \sqrt{\norm{\Context{\Optimalarm{t}}{t}}_{V^{-1}_{t-1}}^2+\norm{\Context{\Action{t}}{t}}_{V^{-1}_{t-1}}^2}.
\label{eq:decomposition}
\end{equation}
By definition, $\Diff{\Action{t}}{t} = \Regret{t}$ and the regret at round $t$ can be expressed as the two additive terms, which presents a stark contrast with multiplicative decomposition of the regret in \citet{agrawal2014thompson}.  
In section \ref{subsec:dim_free_DR_estimator} we show how each term can be bounded with separate rate.

\subsection{Bounds for the cumulative regret}
\label{subsec:dim_free_DR_estimator}
We first bound the leading term of (\ref{eq:decomposition}) and introduce a novel estimation error bound free of $d$ for the contextual bandit DR estimator.
\begin{thm}
(A dimension-free estimation error bound for the contextual bandit
DR estimator.) 
\label{thm:Estimation_error} Suppose Assumptions
1-4 hold. 
For each $t=1,\ldots,T$, let $\Betatrunc{t}$ be any $\History{t}$-measurable estimator satisfying $\|\Betatrunc{t}-\beta\|_{2}\le b$, for some constant $b>0$. 
For each $i$ and $t$, assume that $\SelectionP i{t}>0$ and that there exists $\gamma \in [1/(N+1), 1/N)$ such that $\SelectionP{\Action{t}}{t}>\gamma$.
Given any $\delta\in(0,1)$, set $\lambda_{t}=4\sqrt{2}N\sqrt{t\log\frac{12\tau^{2}}{\delta}}$.
Then with probability at least $1-\delta$, the estimator $\Estimator t$ in (\ref{eq:estimator}) satisfies
\begin{equation}
\norm{\Estimator t-\beta}_{2}\le\frac{C_{b,\sigma}}{\phi^{2}\sqrt{t}}\sqrt{\log\frac{12t^{2}}{\delta}},
\label{eq:estimation_error_bound}
\end{equation}
for all $t=1,\ldots,T$, where the constant $C_{b,\sigma}$ which depends only on $b$ and $\sigma$.
\end{thm}

In bandit literature, estimation error bounds typically include a term involving $d$ which emerges from using the following two Lemmas:
(i) the self-normalized bound for vector-valued martingales \citep[Theorem 1]{abbasi2011improved},
and (ii) the concentration inequality for the covariance matrix \citep[Corollary 5.2]{tropp2015introduction}. 
Instead of using (i) and (ii), we develop the two dimension-free
bounds in Lemmas \ref{lem:eta_x_bound} and  \ref{lem:min_eigenvalue_concentration}, to replace (i) and (ii), respectively. 
With the two Lemmas, we eliminate the dependence on $d$ and express the estimation error bound with $\phi^{2}$ alone.

\begin{lem}
\label{lem:eta_x_bound} (A dimension-free bound for vector-valued martingales.) 
Let $\{\Filtration{\tau}\}_{\tau=1}^{t}$ be a filtration and $\{\eta(\tau)\}_{\tau=1}^{t}$ be a real-valued stochastic process such that $\eta(\tau)$ is $\Filtration{\tau}$-measurable. 
Let $\left\{ X(\tau)\right\} _{\tau=1}^{t}$ be an $\Real^{d}$-valued stochastic process where $X(\tau)$ is $\Filtration{\tau-1}$-measurable and $\norm{X(\tau)}_{2}\le1$. 
Assume that $\{\eta(\tau)\}_{\tau=1}^{t} $ are $\sigma$-sub-Gaussian as in Assumption 2.
Then with probability at least $1-\delta/t^2$, there exists an absolute constant $C>0$ such that
\begin{equation}
\norm{\sum_{\tau=1}^{t}\eta(\tau)X(\tau)}_{2}\le C \sigma\sqrt{t}\sqrt{\log\frac{4t^{2}}{\delta}}.
\label{eq:etaX_bound}
\end{equation}
\end{lem}

Compared to Theorem 1 of \citet{abbasi2011improved},
our bound (\ref{eq:etaX_bound}) does not involve $d$,  yielding a dimension-free bound for vector-valued martingales.
However, the bound (\ref{eq:etaX_bound}) has $\sqrt{t}$ term which comes from using $\norm{\cdot}_{2}$ instead of the self-normalized norm $\norm{\cdot}_{V_{t}^{-1}}$. 

To complete the proof of Theorem \ref{thm:Estimation_error}, we need the following condition,
\begin{equation}
\Mineigen{V_{t}}\ge ct,
\label{eq:min_eigen_problem}
\end{equation}
for some constant $c>0$.
\citet{li2017provably} points out that satisfying (\ref{eq:min_eigen_problem}) is challenging.
To overcome this difficulty, \citet{amani2019linear} and \citet{bastani2020online} use an assumption on the covariance matrix of contexts and a concentration inequality for matrix to prove (\ref{eq:min_eigen_problem}), described as follows.
\begin{prop} 
\label{prop:tropp_matrix_concentration}
\citep[Theorem 5.1.1]{tropp2015introduction}
Let $P(1),$ $\ldots,P(t)\in\Real^{d\times d}$ be the symmetric matrices such that $\lambda_{\min}(P(\tau)) \ge 0$, $\lambda_{\max}(P(\tau))\le L$ and $\lambda_{\min}(\Expectation[P(\tau)])\ge\phi^{2}$, for all $\tau=1,2,\ldots,t$. 
Then,
\begin{equation}
\Probability\left(\Mineigen{\sum_{\tau=1}^{t}P(\tau)}\le\frac{t\phi^{2}}{2}\right)\le d\exp\left(-\frac{t\phi^{2}}{8L}\right).\label{eq:tropp_matrix_concentration}
\end{equation}
\end{prop}

To prove (\ref{eq:min_eigen_problem}) using (\ref{eq:tropp_matrix_concentration}) with probability at least $1-\delta$, for $\delta\in(0,1)$, it requires $t\ge\frac{8L}{\phi^{2}}\log\frac{d}{\delta}$.
Thus, one can use (\ref{eq:tropp_matrix_concentration}) only after $O(\phi^{-2}\log d)$ rounds. 
Due to this requirement, \citet{bastani2020online} implements the forced sampling techniques for $O\left(N^{2}d^{4}(\log d)^{2}\right)$ rounds, and \citet{amani2019linear} forces to select arms randomly for $O\left(\phi^{-2}\log d\right)$ rounds. 
These mandatory exploration phase empirically prevents the algorithm choosing the optimal arm.
An alternative form of matrix Chernoff inequality for adapted sequences is  Theorem 3 in \citet{tropp2011user}, 
but the bound also has a multiplicative factor of $d$. 
Instead of applying Proposition \ref{prop:tropp_matrix_concentration} to prove (\ref{eq:min_eigen_problem}), we utilize a novel dimension-free concentration inequality stated in the following Lemma.

\begin{lem}
\label{lem:min_eigenvalue_concentration} 
(A dimension-free concentration bound for symmetric bounded matrices.) 
Let $\norm A_{F}$ be a Frobenious norm of a matrix $A$. 
Let $\left\{ P(\tau)\right\}_{\tau=1}^{t} \in \Real^{d\times d}$ be the symmetric matrices adapted to a filtration $\{\Filtration{\tau}\}_{\tau=1}^{t}$. 
For each $\tau=1,\ldots,t$, suppose that $\norm{P(\tau)}_{F}\le c$, for some $c>0$ and $\Mineigen{\CE{P(\tau)}{\Filtration{\tau-1}}}\ge\phi^{2}>0$, almost surely. 
For given any $\delta\in(0,1)$, set $\lambda_{t}\ge4\sqrt{2}c\sqrt{t}\sqrt{\log\frac{4t^2}{\delta}}$.
Then with probability at least $1-\delta/t^2$, 
\begin{equation}
\Mineigen{\sum_{\tau=1}^{t}P(\tau)+\lambda_{t}I}\ge\phi^{2}t.
\label{eq:min_eigen_lower_bound}
\end{equation}
\end{lem}

Lemma \ref{lem:min_eigenvalue_concentration} shows that setting $\lambda_{t}$ with $\sqrt{t}$ rate guarantees (\ref{eq:min_eigen_problem}) for all $t\ge1$.
We incorporate  $\lambda_t$ stated in Lemma \ref{lem:min_eigenvalue_concentration} in our estimator (\ref{eq:estimator}), and show in Section \ref{sec:simulation} that the DR estimator regularized with $\lambda_{t}$ outperforms estimators from other contextual bandit algorithms in early rounds.

We obtain the bounds free of $d$ in Lemmas \ref{lem:eta_x_bound} and \ref{lem:min_eigenvalue_concentration}  mainly by applying  Lemma 2.3 in \citet{lee2016} which states that any Hilbert space martingale can be reduced to $\Real^{2}$.
Thus, we can project the vector-valued (or the matrix) martingales to $\Real^{2}$-martingales, and reduce the dimension from $d$ (or $d^{2}$) to $2$. 
Then we apply Azuma-Hoeffding inequality just twice, instead of $d$ times. 
In this way, Lemma \ref{lem:min_eigenvalue_concentration} provides a novel dimension-free bound for the covariance matrix.

Lemmas \ref{lem:eta_x_bound} and \ref{lem:min_eigenvalue_concentration} can be applied to other works to improve the existing bounds.  
For example, using these Lemmas, the estimation error bound of \citet{bastani2020online} can be improved by a factor of $\log d$.
Proposition EC.1 of \citet{bastani2020online} provides an estimation error bound for the ordinary least square estimator by using Proposition \ref{prop:tropp_matrix_concentration} and bounding all values of $d$ coordinates.
By applying Lemmas \ref{lem:eta_x_bound} and \ref{lem:min_eigenvalue_concentration}, one does not have to deal with each coordinate and eliminate dependence on $d$.

Using Lemma \ref{lem:min_eigenvalue_concentration}, we can bound the second term of the regret in (\ref{eq:decomposition}) as follows.
For $j=1,\ldots,N$
\begin{equation}
\norm{\Context{j}{t}}_{V_{t-1}^{-1}} 
\le\norm{\Context jt}_{2}\sqrt{\norm{V_{t-1}^{-1}}_{2} }
\le\Mineigen{V_{t-1}}^{-1/2} 
\le\frac{1}{\sqrt{\phi^{2}N(t-1)}}.
\label{eq:second_term}
\end{equation}
Finally, we are ready to bound $\Regret{t}$ in (\ref{eq:decomposition}).

\begin{lem}
\label{lem:regret_t_bound}
Suppose the assumptions in Theorem \ref{thm:regret_bound} hold.
Then with probability at least $1-2\delta$,
\begin{equation}
\Regret{t} \le \frac{2C_{b,\sigma}}{\phi^{2}\sqrt{t-1}}\sqrt{\log\frac{12t^{2}}{\delta}} + \frac{\sqrt{2}}{\phi \sqrt{N(t-1)}},
\label{eq:super_saturated_regret}
\end{equation}
for all $t=2,\ldots,T$.
\end{lem}

\begin{proof}
Since $\Action{t}$ is shown to be super-unsaturated with high probability, we can use \eqref{eq:decomposition} to have $\Regret t\le 2\| \Estimator{t-1}-\beta\|_{2} + \sqrt{ \|\Context{\Optimalarm{t}}{t}\|_{V_{t-1}^{-1}}^2+\|\Context{\Action{t}}{t}\|_{V_{t-1}^{-1}}^2},
$ for all $t=2,\ldots,T$.
We see that the first term is bounded by Theorem \ref{thm:Estimation_error}, and the second term by (\ref{eq:second_term}).
Note that to prove Theorem 1, Lemma \ref{lem:min_eigenvalue_concentration} is invoked, and the event (\ref{eq:min_eigen_lower_bound}) of Lemma \ref{lem:min_eigenvalue_concentration} is a subset of that in (\ref{eq:estimation_error_bound}).   
Therefore (\ref{eq:super_saturated_regret}) holds with probability at least $1-2\delta$ instead of $1-3\delta$.
Details are given in supplementary materials.
\end{proof}

\begin{figure*}[t]
\begin{center}
\includegraphics[width=0.32\linewidth]{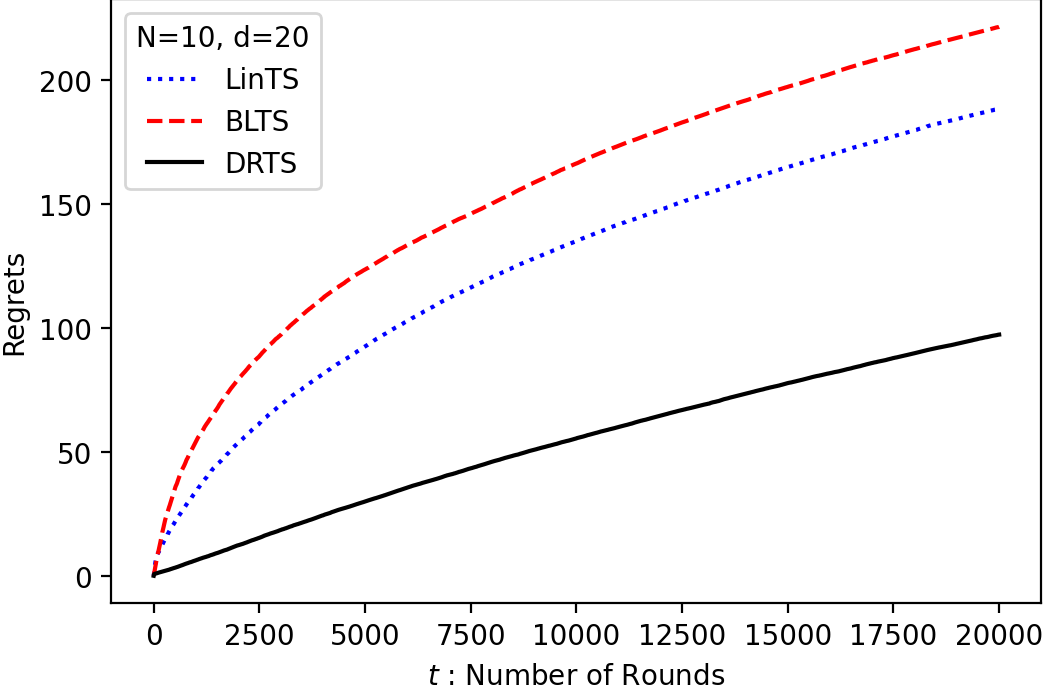}
\includegraphics[width=0.32\linewidth]{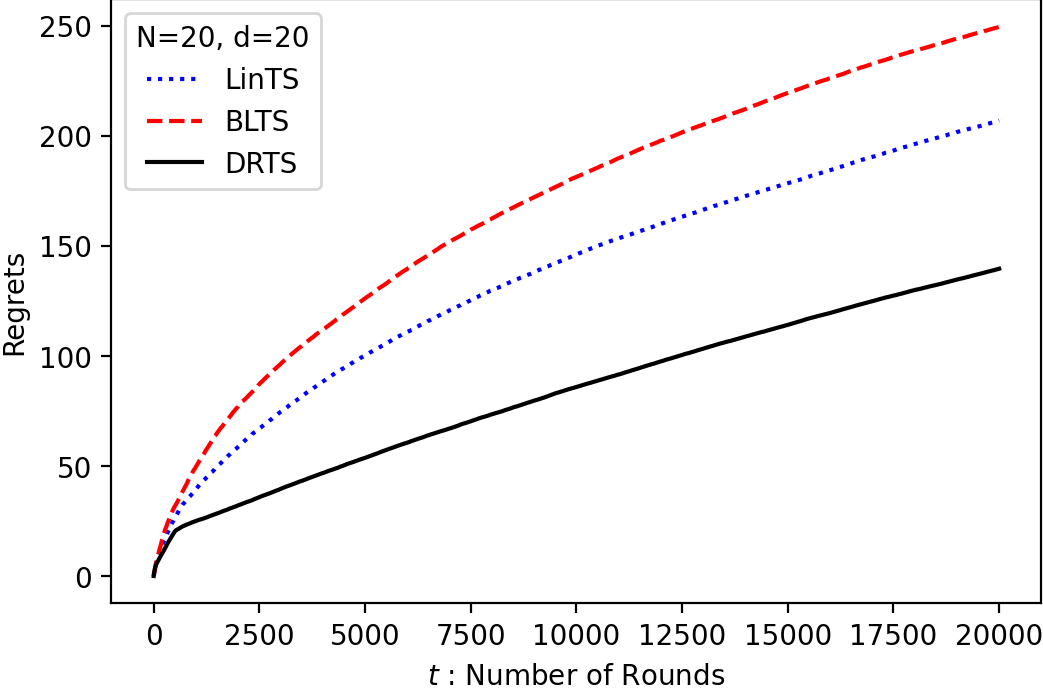}
\includegraphics[width=0.32\linewidth]{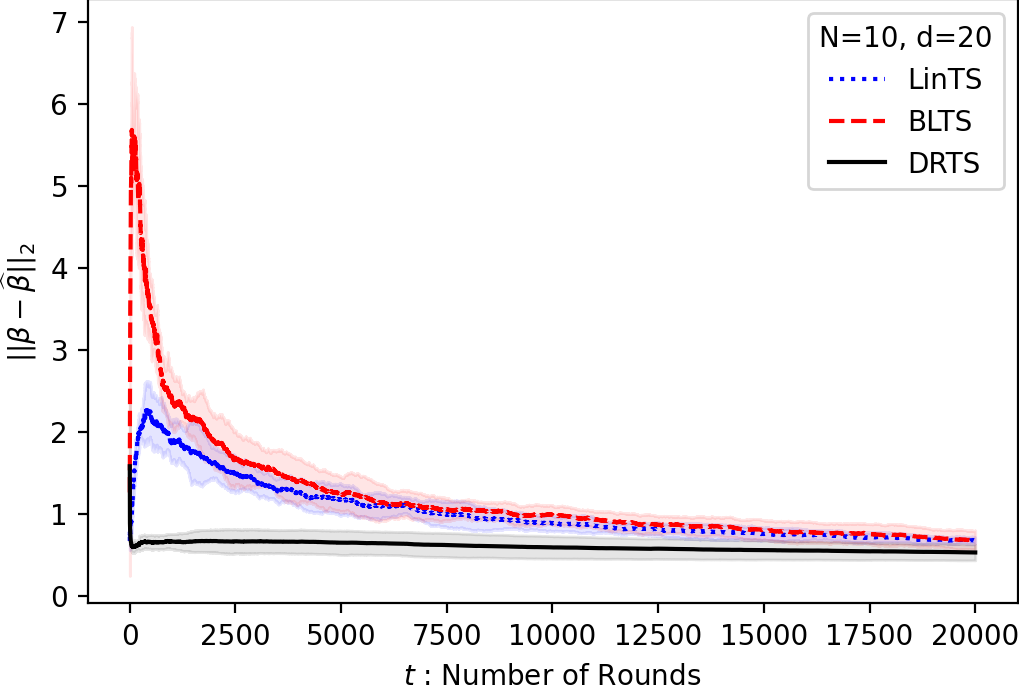}
\\
\includegraphics[width=0.32\linewidth]{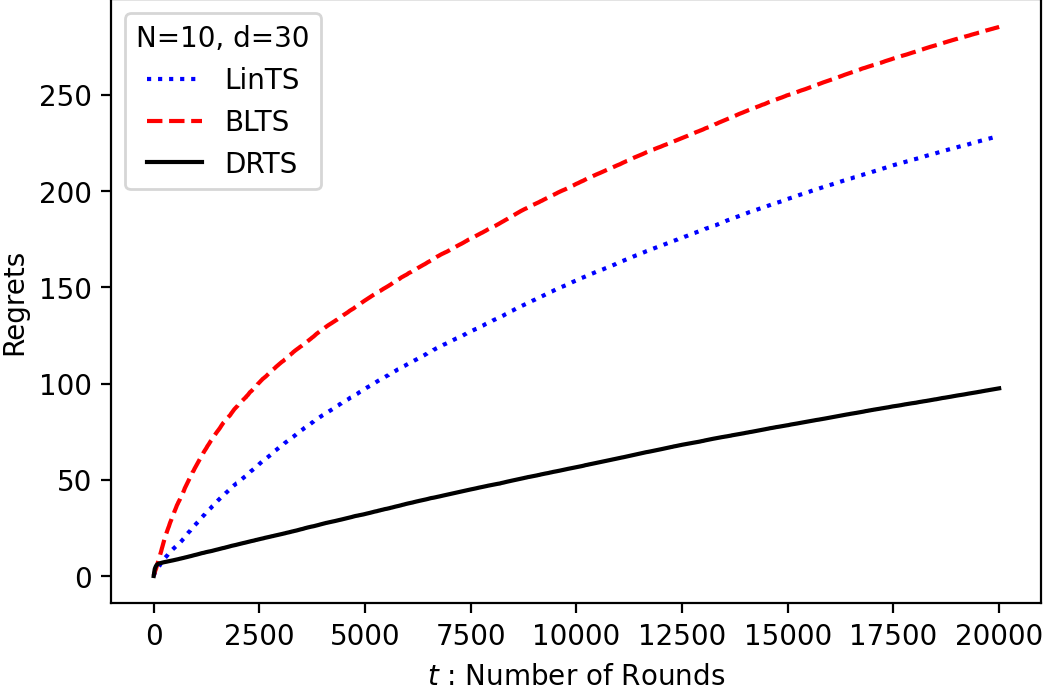}
\includegraphics[width=0.32\linewidth]{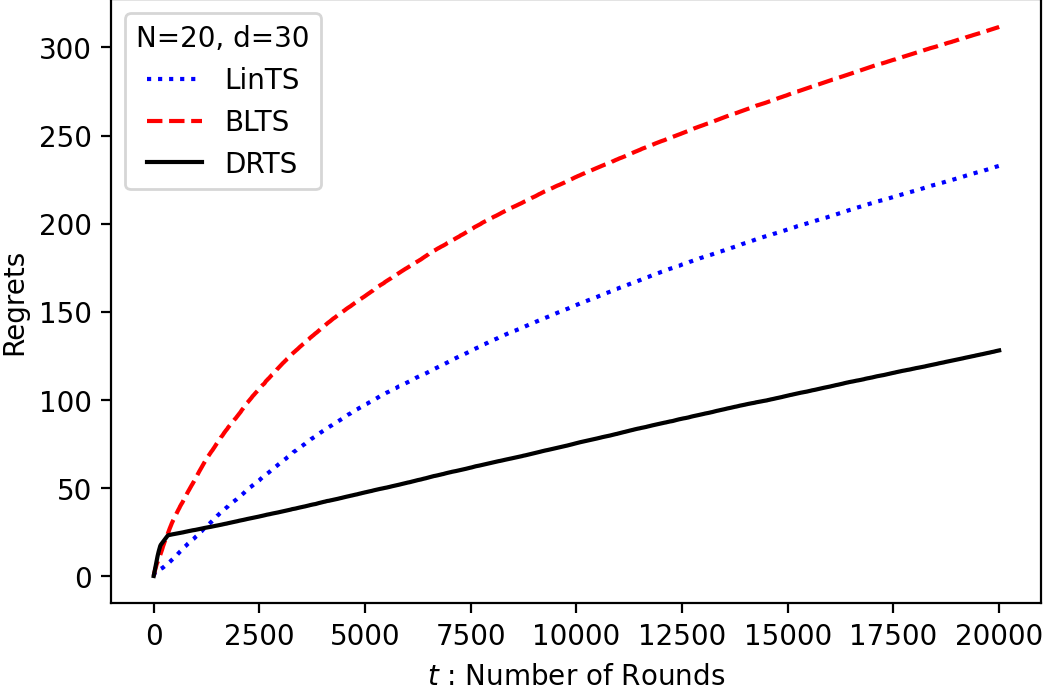}
\includegraphics[width=0.32\linewidth]{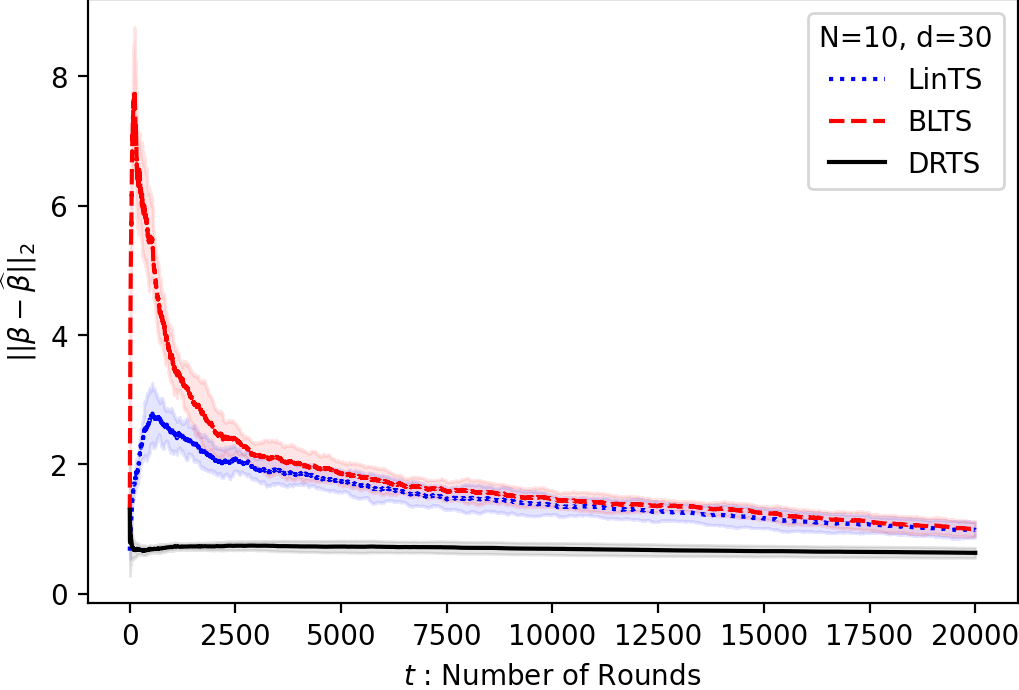}
\caption{
A Comparison of cumulative regrets and estimation errors of \texttt{LinTS}, \texttt{BLTS} and \texttt{DRTS}. 
Each line shows the averaged cumulative regrets (estimation errors, resp.) and the shaded area in the right two figures represents the standard deviations over 10 repeated experiments.
}
\label{fig:results}
\end{center}
\end{figure*}

Lemma \ref{lem:regret_t_bound} shows that the regret at round $t$ does not exceed a $O(\phi^{-2}t^{-1/2})$ bound when $a_t\in N_t$, which is guaranteed in our algorithm via resampling with high probability (See Section \ref{subsec:maximum_possible_resampling} for details).
This concludes the proof of Theorem \ref{thm:regret_bound}.

\section{Simulation studies}
\label{sec:simulation}
In this section, we compare the performances of the three algorithms: (i) \texttt{LinTS} \citep{agrawal2013thompson}, (ii) \texttt{\texttt{BLTS}} \citep{dimakopoulou2019balanced}, and (iii) the proposed \texttt{DRTS}.
We use simulated data described as follows.
The number of arms \(N\) is set to \(10\) or \(20\), and the dimension of contexts \(d\) is set to \(20\) or \(30\).
For each element of the contexts \(j=1,\cdots,d\), we generate \([X_{1j}(t),\cdots,X_{Nj}(t)]\) from a normal distribution \(\mathcal{N}(\mu_N, V_N)\) with mean \(\mu_{10}=[-10,-8,\cdots, -2,2,\cdots,8, 10]^T\), or \(\mu_{20}=[-20,-18,\cdots, -2,2,\cdots,18, 20]^T\), and the covariance matrix \(V_N\in\mathbb{R}^{N\times N}\) has \(V_{N}(i,i)=1\) for every \(i\) and \(V_{N}(i,k)=\rho\) for every \(i\neq k\). 
We set \(\rho=0.5\) and truncate the sampled contexts to satisfy \(\|X_{i}(t)\|_{2} \le 1\). 
To generate the stochastic rewards, we sample \(\eta_i(t)\) independently from \(\mathcal{N}(0,1)\). 
Each element of \(\beta\) follows a uniform distribution, \(\mathcal{U}(-{1}/{\sqrt{d}}, {1}/{\sqrt{d}})\).

All three algorithms have \(v\) as an input parameter which controls the variance of \(\BetaSampled{i}{t}\).
\texttt{BLTS} and \texttt{DRTS} require a positive threshold \(\gamma\) which truncates the selection probability. 
We consider $v \in \{0.001, 0.01, 0.1, 1\}$ in all three algorithms, $\gamma \in \{0.01, 0.05, 0.1\}$ for \texttt{BLTS}, and set $\gamma = 1/(N+1)$ in \texttt{DRTS}.
Then we report the minimum regrets among all combinations.
The regularization parameter is \(\lambda_{t}=\sqrt{t}\) in \texttt{DRTS} and \(\lambda_{t}=1\) in both \texttt{LinTS} and \texttt{BLTS}.
To obtain an imputation estimator $\check{\beta}_t$ required in \texttt{DRTS}, we use ridge regression with $\{\Context{\Action{\tau}}{\tau},\Reward{\Action{\tau}}{\tau}\}_{\tau=1}^{t-1}$, for each round $t$.
Other implementation details are in supplementary materials.

Figure \ref{fig:results} shows the average of the cumulative regrets and the estimation error \(\|\widehat{\beta}_t-\beta\|_2\) of the three algorithms based on 10 replications.
The figures in the two left columns show the average cumulative regret according to the number of rounds with the best set of hyperparameters for each algorithm.
The total rounds are \(T=20000\).
The figures in the third columns show the average of the estimation error \(\|\widehat{\beta}_t-\beta\|_{2}\).
In the early stage, the estimation errors of \texttt{LinTS} and \texttt{BLTS} increase rapidly, while that of \texttt{DRTS} is stable.
The stability of the DR estimator follows possibly by using full contexts and the regularization parameter \(\lambda_{t}=\sqrt{t}\).
This yields a large margin of estimation error among \texttt{LinTS}, \texttt{BLTS} and \texttt{DRTS}, especially when the dimension is large.

\section{Conclusion}
\label{sec:conclusion}
In this paper, we propose a novel algorithm for stochastic contextual linear bandits.
Viewing the bandit problem as a missing data problem, we use the DR technique to employ all contexts including those that are not chosen.
With the definition of super-unsaturated arms, we show a regret bound which only  depends on the minimum eigenvalue of the sample covariance matrices.
This new bound has $\tilde{O}(d\sqrt{T})$ rate in many practical scenarios, which is improved by a factor of $\sqrt{d}$ compared to the previous \texttt{LinTS} regret bounds.
Simulation studies show that the proposed algorithm performs better than other \texttt{LinTS} algorithms in a large dimension.

\section*{Acknowledgements}
This work is supported by the National Research Foundation of Korea (NRF) grant funded by the Korea government (MSIT, No.2020R1A2C1A01011950) (Wonyoung Kim and Myunghee Cho Paik), and by the Institute of Information \& communications Technology Planning \& Evaluation (IITP) grant funded by the Korea government (MSIT) (No.2020-0-01336, Artificial Intelligence Graduate School Program(UNIST)) and the National Research Foundation of Korea (NRF) grant funded by the Korea government (MSIT, No.2021R1G1A100980111) (Gi-Soo Kim).
Wonyoung Kim was also supported by Hyundai Chung Mong-koo foundation.

\bibliography{ref}
\bibliographystyle{plainnat}

\clearpage
\appendix

\section*{Checklist}
\begin{enumerate}

\item For all authors...
\begin{enumerate}
  \item Do the main claims made in the abstract and introduction accurately reflect the paper's contributions and scope?
    \answerYes{}
  \item Did you describe the limitations of your work?
    \answerYes{See supplementary materials.}
  \item Did you discuss any potential negative societal impacts of your work?
    \answerNA{Since this work focuses on the theoretical regret bound of an algorithm, the societal impacts of the results are too indirect and too broad to discuss.}
  \item Have you read the ethics review guidelines and ensured that your paper conforms to them?
    \answerYes{}
\end{enumerate}

\item If you are including theoretical results...
\begin{enumerate}
  \item Did you state the full set of assumptions of all theoretical results?
    \answerYes{See Section \ref{subsec:Settings_and_assumptions}}
	\item Did you include complete proofs of all theoretical results?
    \answerYes{See supplementary materials}
\end{enumerate}

\item If you ran experiments...
\begin{enumerate}
  \item Did you include the code, data, and instructions needed to reproduce the main experimental results (either in the supplemental material or as a URL)?
    \answerYes{See supplementary materials}
  \item Did you specify all the training details (e.g., data splits, hyperparameters, how they were chosen)?
    \answerYes{See Section \ref{sec:simulation}}
	\item Did you report error bars (e.g., with respect to the random seed after running experiments multiple times)?
    \answerYes{Only in the estimation error results in Figure \ref{fig:results} for good visualization.}
	\item Did you include the total amount of compute and the type of resources used (e.g., type of GPUs, internal cluster, or cloud provider)?
    \answerNo{The computation in our experiments is too light to characterize and our work focuses on the theoretical analysis of our algorithm.}
\end{enumerate}

\item If you are using existing assets (e.g., code, data, models) or curating/releasing new assets...
\begin{enumerate}
  \item If your work uses existing assets, did you cite the creators?
    \answerNA{}
  \item Did you mention the license of the assets?
    \answerNA{}
    \item Did you include any new assets either in the supplemental material or as a URL?
    \answerNA{}
  \item Did you discuss whether and how consent was obtained from people whose data you're using/curating?
    \answerNA{}
  \item Did you discuss whether the data you are using/curating contains personally identifiable information or offensive content?
    \answerNA{}
\end{enumerate}

\item If you used crowdsourcing or conducted research with human subjects...
\begin{enumerate}
  \item Did you include the full text of instructions given to participants and screenshots, if applicable?
    \answerNA{}
  \item Did you describe any potential participant risks, with links to Institutional Review Board (IRB) approvals, if applicable?
    \answerNA{}
  \item Did you include the estimated hourly wage paid to participants and the total amount spent on participant compensation?
    \answerNA{}
\end{enumerate}

\end{enumerate}

\clearpage
\section{Detailed analysis of the resampling}
In this section, we give details about the issues which can be raised from the resampling in our algorithm.
\subsection{Precise definition of action selection}
We give precise definition of the action at round $t$, $\Action{t}$.
For each round $t\ge2$, given $\History{t}$, let $\Sampledaction{t}{1}, \Sampledaction{t}{2},\ldots,\Sampledaction{t}{M_{t}}$ to be maximum possible sequence of actions to be resampled.
These actions are IID, with $\CP{\Sampledaction t1=i}{\History{t}}=\Pisampled it$ for $i=1,\ldots,N$.
Define a subset of arms $\Gammaset{t}:=\{i:\Pisampled{i}{t}>\gamma\}$ and
a stopping time 
\begin{equation}
\mathcal{T}:=\inf\{m \ge 1: \Sampledaction{t}{m} \in \Gammaset{t} \}
\label{eq:stopping_time}
\end{equation} 
with respect to the filtration $\Filtration{m}:=\History{t}\cup\{\Sampledaction t1,\ldots,\Sampledaction tm\}$.
Since the algorithm stops resampling when the candidate action is in $\Gammaset{t}$, the stopping time $\mathcal{T}$ is the actual number of resampling in algorithm.
Thus we can write the action after resampling as $\Action{t}:=\Sampledaction{t}{\min\{\mathcal{T},M_{t}\}}$.

\subsection{Computing the probability of selection}
The probability of selection $\SelectionP{i}{t}:=\CP{\Action{t}=i}{\History{t}}$ is not the same as $\Pisampled{i}{t}$ due to resampling.
This might cause the problem of computing $\SelectionP{i}{t}$ which is essential to compute $\DRreward{i}{t}$.
However, with the precise definition of $\Action{t}$, we can derive a closed form for $\SelectionP{i}{t}$.

First, we consider two cases separately: (i) the case when the resampling succeeds and (ii) the case when the resampling fails and the maximum possible number of resampling runs out.
In case (i), $\Action{t}\in\Gammaset{t}$, and for any $i\in\Gammaset{t}$, we have
\begin{equation}
\begin{split}
\CP{\Action t=i}{\History t}
=&\CP{\mathcal{T}\le M_{t},\;\Sampledaction{t}{\mathcal{T}}=i}{\History t}\\
=&\sum_{m=1}^{M_{t}}\CP{\mathcal{T}=m,\;\Sampledaction tm=i}{\History t}\\
=&\sum_{m=1}^{M_{t}}\CP{\Sampledaction tm=i}{\History t}\left(\prod_{j=0}^{m-1}\CP{\Sampledaction tj\notin\Gammaset t}{\History t}\right)\\
=&\Pisampled it\sum_{m=1}^{M_{t}}\left(1-\sum_{i\in\Gammaset t}\Pisampled it\right)^{m-1}\\
=&\Pisampled it \frac{1-\left(1-\sum_{i\in\Gammaset t}\Pisampled it \right)^{M_t}}{\sum_{i\in\Gammaset t}\Pisampled it}.
\end{split}
\label{eq:pi_computation_success}
\end{equation}

Now, for the case (ii) $\Action{t}\notin\Gammaset{t}$, and for any $i\notin\Gammaset{t}$, we have
\begin{equation}
\begin{split}
\CP{a_{t}=i}{\History t}&=\CP{\mathcal{T}>M_{t},\Sampledaction t{M_{t}}=i}{\History t}\\
&=\CP{\bigcap_{m=1}^{M_{t}-1}\left\{ \Sampledaction tm\notin\Gammaset t\right\} ,\Sampledaction t{M_{t}}=i}{\History t}\\
&=\left(1-\sum_{i\in\Gammaset t}\Pisampled it\right)^{M_{t}-1}\hspace{-13pt}\Pisampled it.
\end{split}
\label{eq:pi_computation_fail}
\end{equation}
With (\ref{eq:pi_computation_success}) and (\ref{eq:pi_computation_fail}), we can compute $\SelectionP{i}{t}$ for all $i=1,\ldots,N$.

\subsection{The number of maximum possible resampling}
\label{subsec:maximum_possible_resampling}
The proposed algorithm attempts resampling up to $M_t$ times to find an arm in $\{i:\tilde{\pi}_i(t)>\gamma\}$.  
The main point in selecting $M_t$ is to bound the probability that the resampling fails in finding an arm whose selection probability exceeds $\gamma$ for some $\delta$, i.e., 
\begin{equation}
\label{eq:resampling_probability_bound}
\mathbb{P}(a_t\notin\{i:\tilde{\pi}_i(t) > \gamma\})\le\delta/t^2.
\end{equation}
Intuitively, as $M_t$ increases, we have more opportunities for resampling and the probability that the resampling fails in finding arms in $\{i:\tilde{\pi}_i(t)>\gamma\}$ decreases.
Since $\gamma < 1/N$, there exists $j$ such that $\tilde{\pi}_j(t)>\gamma$, and the probability that the resampling fails is less than $1-\gamma$ in each resampling trial.

Specifically, we can achieve \eqref{eq:resampling_probability_bound} by choosing $M_t$ as a minimum integer that exceeds $\log\frac{t^2}{\delta}/\log\frac{1}{1-\gamma}$.
For any given $\delta\in(0,1)$, the event $\{\Action{t}\in \Gammaset{t}\}$ occurs with probability at least $1-\delta/t^2$.
By (\ref{eq:stopping_time}), we have
\begin{equation*}
\CP{\Action t\notin\Gammaset t}{\History t}=\CP{\mathcal{T}>M_{t}}{\History t}=\CP{\bigcap_{m=1}^{M_{t}}\left\{ \Sampledaction tm\notin\Gammaset t\right\} }{\History t}=\left(1-\sum_{i\in\Gammaset t}\Pisampled it\right)^{M_{t}}.
\end{equation*}
Since $\gamma < 1/N$, there exists at least one arm in $\Gammaset{t}$, and thus
$\CP{\Action t\notin\Gammaset t}{\History t}\le\left(1-\gamma\right)^{M_{t}}$.
If we set $M_{t}$ as a minimum integer that exceeds $\left(\log\frac{t^{2}}{\delta}\right)\left(\log\frac{1}{1-\gamma}\right)^{-1}$ then \eqref{eq:resampling_probability_bound} holds.
Thus, by choosing $M_t$ for each round that satisfies \eqref{eq:resampling_probability_bound}, the algorithm finds an arm $j$ such that $\tilde{\pi}_j(t)>\gamma$ in all rounds with high probability.

Selecting an arm from the set $\{i:\tilde{\pi}_i(t)>\gamma\}$ with high probability is crucial in achieving the regret bound of order $\tilde{O}(\phi^{-2}\sqrt{T})$ for two reasons.
First, it guarantees that the arm is super-unsaturated and our novel regret decomposition \eqref{eq:decomposition} holds to achieve our regret bound. 
Let $N_t$ be the set of super-unsaturated arm defined in (\ref{eq:super_unsaturated_set}).
With Lemma \ref{lem:super_unsaturated_arms}, we prove that if $\Pisampled{i}{t}>\gamma$ then $i\in N_{t}$, which implies $\Gammaset{t} \subseteq N_t$, and thus
\begin{equation*}
\CP{\Action{t} \in N_{t} }{\History{t}} 
\ge \CP{\Action{t} \in \Gammaset{t} }{\History{t}}.
\end{equation*}
Thus we can conclude that $\Action{t}$ is super-unsaturated with probability at least $1-\delta/t^2$ with $M_{t}$ defined in Section \ref{subsec:maximum_possible_resampling}.
Second, the inverse probability, $\pi_{a_t}(t)^{-1}$ is bounded by $\gamma^{-1}$ which appears in $Y^{DR}_i(t)$ and the proof of Theorem \ref{thm:Estimation_error}.
From \eqref{eq:pi_computation_success} we can deduce $\pi_{a_t}(t) \ge \tilde{\pi}_{a_t}(t) > \gamma$, for $a_t\in\Gammaset{t}$.
This shows that the assumptions regarding $\SelectionP{\Action{t}}{t}$ in Theorem \ref{thm:Estimation_error} hold.

\section{Technical Lemmas}
\begin{lem}
\label{lem:Bernstein_Concentration} \citep[Theorem 2.19]{wainwright_2019}
(Bernstein Concentration) Let
$\{D_{k},\mathfrak{S}_{k}\}_{k=1}^{\infty}$ be a martingale difference
sequence and suppose $D_{k}$ is $\sigma$-sub-Gaussian in an adapted
sense, i.e. for all $\lambda\in\Real$, $\CE{e^{\lambda D_{k}}}{\mathfrak{S}_{k-1}}\le e^{\lambda^{2}\sigma^{2}/2}$
almost surely. Then for all $x\ge0$, 
\[
\Probability\left(\abs{\sum_{k=1}^{n}D_{k}}\ge x\right)\le2\exp\left(-\frac{x^{2}}{2n\sigma^{2}}\right).
\]
\end{lem}

\begin{lem}
\label{lem:Azuma_Hoeffding_ineqaulity} \citep{azuma1967weighted}
(Azuma-Hoeffding inequality) If a super-martingale $(Y_{t};t\ge0)$
corresponding to filtration $\Filtration t$, satisfies $\abs{Y_{t}-Y_{t-1}}\le c_{t}$
for some constant $c_{t}$, for all $t=1,\ldots,T$, then for any
$a\ge0$, 
\[
\Probability\left(Y_{T}-Y_{0}\ge a\right)\le e^{-\frac{a^{2}}{2\sum_{t=1}^{T}c_{t}^{2}}}.
\]
\end{lem}

\begin{lem}
\label{lem:dim_reduction}\citep[Lemma 2.3]{lee2016}  Let $\left\{ N_{t}\right\} $
be a martingale on a Hilbert space $(\mathcal{H},\norm{\cdot}_{\mathcal{H}})$.
Then there exists a $\Real^{2}$-valued martingale $\left\{ P_{t}\right\} $
such that for any time $t\ge0$, $\norm{P_{t}}_{2}=\norm{N_{t}}_{\mathcal{H}}$
and $\norm{P_{t+1}-P_{t}}_{2}=\norm{N_{t+1}-N_{t}}_{\mathcal{H}}$.
\end{lem}

\begin{lem}
\label{lem:chung_lemma}
\citep[Lemma 1, Theorem 32]{chung2006concentration}
For a filtration $\Filtration 0\subset\Filtration 1\subset\cdots\subset\Filtration T$, suppose each random variable $X_{t}$ is $\Filtration t$-measurable martingale, for $0\le t\le T$. 
Let $B_{t}$ denote the bad set associated with the following admissible condition: 
\[
\abs{X_{t}-X_{t-1}}\le c_{t},
\] 
for $1\le t\le T$, where $c_{1},\ldots,c_{n}$ are non-negative numbers. 
Then there exists a collection of random variables $Y_{0},\ldots,Y_{T}$ such that $Y_{t}$ is $\Filtration t$-measurable martingale such that 
\[
\abs{Y_{t}-Y_{t-1}}\le c_{t},
\]
and $\{\omega:Y_{t}(\omega)\neq X_{t}(\omega)\}\subset B_{t}$, for $0\le t\le T$.  
\end{lem}

\begin{rem}
We found a counter example where Lemma~\ref{lem:chung_lemma} does not hold.
Suppose for each $t\in[T]$, $\Filtration{t}:=\{X_1,\ldots,X_t\}$ and
\[
X_{t}-X_{t-1}:=
\begin{cases}
2-\frac{{t+1}^{2}}{\delta} & \text{with probability }\frac{\delta}{{t+1}^{2}},\\
2-\frac{1}{1-\frac{\delta}{{t+1}^{2}}} & \text{with probability }1-\frac{\delta}{{t+1}^{2}}.
\end{cases}
\]
Then $X_t$ is $\Filtration{t}$ measurable martingale  with $X_0=0$.
Set $B_t:=\{|X_t-X_{t-1}|>1\}$.
By Lemma~\ref{lem:chung_lemma}, there exists a $\Filtration{t}$-measurable martingale $Y_t$ such that $|Y_t-Y_{t-1}|<1$ and
$\{\omega:Y_{t}(\omega)\neq X_{t}(\omega)\}\subset B_{t}$, for $0\le t\le T$.
By Lemma~\ref{lem:Bernstein_Concentration},
\begin{align*}
\Probability\left(X_{T}\ge\sqrt{2T\log\frac{1}{\delta}}\right)
&\le\Probability\left(\left\{ X_{T}\ge\sqrt{2T\log\frac{1}{\delta}}\right\} \cap B_{T}^{c}\right)+\Probability\left(B_{T}\right) \\
&\le\Probability\left(Y_{T}\ge\sqrt{2T\log\frac{1}{\delta}}\right)+\sum_{t=1}^{T}\frac{\delta}{{t+1}^{2}}\le2\delta.
\end{align*}
However, by definition of $X_t$,
\begin{align*}
\Probability\left(X_{T}\ge\sqrt{2T\log\frac{1}{\delta}}\right)
&\ge\Probability\left(X_{T}\ge T\right)\\
&\ge\Probability\left(X_{T}=2T-\sum_{t=1}^{T}\frac{1}{1-\frac{\delta}{\left(t+1\right)^{2}}}\right)\\
&=1-\sum_{t=1}^{T}\frac{\delta}{\left(t+1\right)^{2}}\ge1-\delta,
\end{align*}
holds for $\delta \ge e^{-T/2}$, which is a contradiction to the first inequality.
\label{rem:chung_lemma}
\end{rem}

\begin{lem}
\label{lem:subgaussian} Suppose a random variable $X$ satisfies
$\Expectation[X]=0$, and let $\eta$ be an $\sigma$-sub-Gaussian random variable. 
If $\abs X\le\abs{\eta}$ almost surely, then $X$ is $C\sigma$-sub-Gaussian
for some absolute constant $C>0$.
\end{lem}

\begin{proof}
By Proposition 2.5.2 in \citet{vershynin2018high}, there exists an absolute constant $C_{1}>0$ such that
\[
\Expectation\exp\left(\lambda^{2}\eta^{2}\right)\le\exp\left(\frac{\lambda^{2}C_{1}^{2}\sigma^2}{2}\right),\quad\forall\lambda\in\left[-\frac{\sqrt{2}}{C_{1}\sigma},\frac{\sqrt{2}}{C_{1}\sigma}\right].
\]
Since $\abs X\le\abs{\eta}$ almost surely,
\[
\Expectation\exp\left(\lambda^{2}X^{2}\right)\le\exp\left(\frac{\lambda^{2}C_{1}^{2}\sigma^2}{2}\right),\quad\forall\lambda\in\left[-\frac{\sqrt{2}}{C_{1}\sigma},\frac{\sqrt{2}}{C_{1}\sigma}\right].
\]
Since $\Expectation[X]=0$, by Proposition 2.5.2 in \citet{vershynin2018high}, there exists an absolute constant $C_{2}>0$ such that
\[
\Expectation\exp\left(\lambda X\right)\le\exp\left(\frac{\lambda^{2}C_{1}^{2}C_{2}^{2}\sigma^2}{2}\right),\quad\forall\lambda\in\Real.
\]
Setting $C=C_{1}C_{2}$ completes the proof.
\end{proof}

\section{Missing details in proof of Theorem \ref{thm:regret_bound}}
In section \ref{subsec:maximum_possible_resampling}, we prove that $\Action{t}\in \Gammaset{t}$ with probability at least $1-\delta/t^2$, for all $t\ge2$.
Thus, for any $x>0$,
\begin{equation*}
\begin{split}
\Probability\left(R(T)>x\right)\le&\Probability\left(R(T)>x,\;\bigcap_{t=2}^{T}\left\{ \Action t\in\Gammaset t\right\} \right)+\Probability\left(\bigcup_{t=2}^{T}\left\{ \Action t\notin\Gammaset t\right\} \right)\\
\le&\Probability\left(R(T)>x,\;\bigcap_{t=2}^{T}\left\{ \Action t\in\Gammaset t\right\} \right)+\delta\\
\le&\Probability\left(2+\sum_{t=2}^{T}\Regret t>x,\;\bigcap_{t=2}^{T}\left\{ \Action t\in\Gammaset t\right\} \right)+\delta
\end{split}
\end{equation*}
The last inequality holds by Assumption 1.
Since $\Gammaset{t}$ is a subset of $N_t$ and by (\ref{eq:decomposition}), 
\begin{equation}
\begin{split}
&\Probability\left(R(T)>x\right)\\
&\le
\Probability\left(2+\sum_{t=2}^{T}\left\{ 2\norm{\Estimator{t-1}-\beta}_{2}\!+\!\sqrt{\norm{\Context{\Optimalarm{t}}{t}}_{V_{t-1}^{-1}}^{2}+\norm{\Context{\Action{t}}{t}}_{V_{t-1}^{-1}}^{2}}\;\right\} >x,\;\bigcap_{t=2}^{T}\left\{ \Action t\in\Gammaset t\right\} \right)+\delta.
\end{split}
\label{eq:cumul_regret_prob}
\end{equation}
To bound the term $\norm{\Estimator{t}-\beta}_{2}$ for all $t=1,\ldots,T-1$, we use Theorem \ref{thm:Estimation_error}.
Before that, we need to verify whether the two assumptions on $\SelectionP{i}{t}$ in Theorem \ref{thm:Estimation_error} hold.

First, we show that $\SelectionP{\Action{t}}{t}>\gamma$.
When $t=1$, we have $\Pisampled{i}{1}=1/N$ for all $i$.
Since $\gamma < 1/N$, we do not need resampling and thus $\SelectionP{i}{t}=\Pisampled{i}{t}>\gamma$.
When $t\ge2$, $\Action{t}\in\Gammaset{t}$ is already concerned in (\ref{eq:cumul_regret_prob}), and thus $\Pisampled{\Action{t}}{t}>\gamma$.
From (\ref{eq:pi_computation_success}), we can deduce that $\SelectionP{i}{t}>\Pisampled{i}{t}$ for all $i\in\Gammaset{t}$, and thus $\SelectionP{\Action{t}}{t}>\gamma$.

Now, we prove that $\SelectionP{i}{t}>0$ for all $i$ and $t$.
The case of $t=1$ is already proved above.
When $t\ge2$, from \eqref{eq:pi_computation_success}, we have 
\[
\SelectionP it:=\CP{\Action t=i}{\History t}=\Pisampled it\sum_{m=1}^{M_{t}}\left(1-\sum_{i\in\Gammaset t}\Pisampled it\right)^{m-1}>\Pisampled it > \gamma,
\]
for all $i\in\Gammaset{t}$.
If there exists an arm $i\notin\Gammaset{t}$, from \eqref{eq:pi_computation_fail},
\begin{equation*}
\SelectionP{i}{t} = \left(1-\sum_{i\in\Gammaset t}\Pisampled it\right)^{M_{t}-1}\Pisampled it.
\end{equation*}
The first term is positive since there exists an arm $i\notin\Gammaset{t}$.
The second term is also positive since the distribution of $\BetaSampled{i}{t}$ has support $\Real^d$, which implies that
\begin{equation*}
\Pisampled it:=\CP{\Context it^{T}\BetaSampled it=\max_{j}\Context jt^{T}\BetaSampled jt}{\History t}>0,
\end{equation*}
for all $i$.
Thus, $\SelectionP{i}{t}>0$ for all $i$ and $t$.
This implies that the two assumptions on $\SelectionP{i}{t}$ in Theorem \ref{thm:Estimation_error} hold.
 
Now we can use Theorem \ref{thm:Estimation_error} and Lemma \ref{lem:min_eigenvalue_concentration} to have
\begin{equation*}
\norm{\Estimator{t-1}-\beta}_{2} \le\frac{C_{b,\sigma}}{\phi^{2}\sqrt{t-1}}\sqrt{\log\frac{12(t-1)^{2}}{\delta}},\quad\!\!\sqrt{\norm{\Context{\Optimalarm{t}}{t}}_{V_{t-1}^{-1}}^{2}+\norm{\Context{\Action{t}}{t}}_{V_{t-1}^{-1}}^{2}}\le\frac{1}{\phi\sqrt{N(t-1)}},
\end{equation*}
for all $t=2,\ldots,T$ with probability at least $1-\delta$.
Thus, setting 
\[
x=2+\frac{4C_{b,\sigma}}{\phi^{2}}\sqrt{T\log\frac{12T^{2}}{\delta}}+\frac{2\sqrt{T}}{\phi\sqrt{N}}
\]
in \eqref{eq:cumul_regret_prob} proves the result.

\section{Proof of Lemma \ref{lem:super_unsaturated_arms}}
\begin{proof}
First, we bring attention to the fact that the optimal arm $\Optimalarm{t}$ is in $N_t$ by definition.
Suppose that the estimated reward of the optimal arm, $\Sampledreward{\Optimalarm{t}}{t}$ is greater than $\Sampledreward{j}{t}$ for all $j \notin N_t$.
In this case, any arm $j\notin N_t$ cannot be the $m_t:=\arg\max_{i} \Sampledreward{i}{t}$.
Then we have
\begin{equation*}
\begin{split}
\CP{m_{t}\in N_t}{\History{t}} & \ge \CP{\Sampledreward{\Optimalarm{t}}{t} > \Sampledreward{j}{t}, \forall j \notin N_t}{\History{t}} \\
& = \CP{Z_{j}(t) > \{\Context{j}{t} - \Context{\Optimalarm{t}}{t}\}^{T}\Estimator{t-1}, \forall j \notin N_t}{\History{t}},
\end{split}
\end{equation*}
where $Z_j(t):=\Sampledreward{\Optimalarm{t}}{t} -\Sampledreward{j}{t} -\{\Context{\Optimalarm{t}}{t} - \Context{j}{t}\}^{T}\Estimator{t-1}$.
Note that $Z_j(t)$ is a Gaussian random variable with mean 0 and variance $v^2(\| \Context{\Optimalarm{t}}{t}\|_{V_{t-1}^{-1}}^2 + \| \Context{j}{t}\|_{V_{t-1}^{-1}}^2)$ given $\History{t}$.
For all $j\notin N_t$,
\begin{align*}
\{\Context{j}{t} - \Context{\Optimalarm{t}}{t}\}^{T}\Estimator{t-1} & = \{\Context{j}{t}- \Context{\Optimalarm{t}}{t}\}^{T}\{\Estimator{t-1} -\beta\} - \Diff{j}{t} \\
& \le 2\norm{\Estimator{t} - \beta}_2 - \Diff{j}{t}  \le -\sqrt{\norm{\Context{\Optimalarm{t}}{t}}_{V_{t-1}^{-1}}^{2}+\norm{\Context{j}{t}}_{V_{t-1}^{-1}}^{2}}.
\end{align*}
The last inequality is due to $j\notin N_t$.
Thus, we can conclude that
\begin{equation*}
\begin{split}
\CP{m_{t}\in N_t}{\History{t}} \ge & \CP{ \frac{Z_{j}(t)}{v\sqrt{\norm{\Context{\Optimalarm{t}}{t}}_{V_{t-1}^{-1}}^{2}+\norm{\Context{j}{t}}_{V_{t-1}^{-1}}^{2}}} > -\frac{1}{v}, \forall j \notin N_t}{\History{t}} \\
:=&\CP{Y_{j}>-v^{-1},\forall j\neq N_{t}}{\History t}.
\end{split}
\end{equation*}
Using the fact that 
$$
Y_{j}:=\frac{Z_{j}(t)}{v\sqrt{\norm{\Context{\Optimalarm{t}}{t}}_{V_{t-1}^{-1}}^{2}+\norm{\Context{j}{t}}_{V_{t-1}^{-1}}^{2}}}
$$ is a standard Gaussian random variable given $\History t$, we have
$$
\CP{Y_{j}\le-v^{-1}}{\History t} \le \exp\left(-\frac{1}{2v^2}\right).
$$
Setting $v=\{2\log(N/(1-\gamma N) )\}^{-1/2}$ gives
$$
\CP{Y_{j}\le-v^{-1}}{\History t} \le  \exp\left(-\log\left(N/(1-\gamma N)\right)\right) =  \frac{1-\gamma N}{N}.
$$
Thus,
\begin{align*}
\CP{m_{t}\in N_{t}}{\History t}\ge & 1-\CP{Y_{j}\le-v^{-1},\exists j\neq N_{t}}{\History t}\\
\ge & 1-\sum_{j\neq N_{t}}\CP{Y_{j}<-v^{-1}}{\History t}\\
\ge & 1-(1-\gamma N) \\
= & \gamma N \\
\ge & 1-\gamma.
\end{align*}
The last inequality holds due to $\gamma\ge1/(N+1)$.
\end{proof}

\section{Proof of Theorem \ref{thm:Estimation_error}}
\begin{proof}
Fix $t=1,\ldots,T$ and let $V_{t}:=\sum_{\tau=1}^{t}\sum_{i=1}^{N}\Context i{\tau}\Context i{\tau}^{T}+\lambda_{t}I$.
For each $i$ and $\tau$, let $\DRError i{\tau}=\DRreward i{\tau}-\Context i{\tau}^{T}\beta$.
Then 
\[
\Estimator t=\beta+V_{t}^{-1}\left(-\lambda_{t}\beta+\sum_{\tau=1}^{t}\sum_{i=1}^{N}\DRError i{\tau}\Context i{\tau}\right).
\]
To bound $\norm{\Estimator t-\beta}_{2}$, 
\begin{align*}
\norm{\Estimator t-\beta}_{2}= & \norm{V_{t}^{-1}\left(-\lambda_{t}\beta+\sum_{\tau=1}^{t}\sum_{i=1}^{N}\DRError i{\tau}\Context i{\tau}\right)}_{2}\\
\le & \norm{V_{t}^{-1}}_{2}\norm{\left(-\lambda_{t}\beta+\sum_{\tau=1}^{t}\sum_{i=1}^{N}\DRError i{\tau}\Context i{\tau}\right)}_{2}\\
= & \left\{ \Mineigen{V_{t}}\right\} ^{-1}\norm{\left(-\lambda_{t}\beta+\sum_{\tau=1}^{t}\sum_{i=1}^{N}\DRError i{\tau}\Context i{\tau}\right)}_{2}.
\end{align*}
By Assumption 1, $\norm{\beta}_{2}\le1$. Using triangle inequality,
\begin{equation}
\norm{\Estimator t-\beta}_{2}\le\left\{ \Mineigen{V_{t}}\right\} ^{-1}\lambda_{t}+\left\{ \Mineigen{V_{t}}\right\} ^{-1}\norm{\sum_{\tau=1}^{t}\sum_{i=1}^{N}\DRError i{\tau}\Context i{\tau}}_{2}.\label{eq:estimation_error_two_terms}
\end{equation}

We will bound the first term in (\ref{eq:estimation_error_two_terms}).
Let $\text{Tr}(A)$ be the trace of a matrix $A$. By the definition
of the Frobenious norm, for $\tau=1,\ldots,t$, and for $i=1,\ldots,N$,
\[
\norm{\sum_{i=1}^{N}\Context i{\tau}\Context i{\tau}^{T}}_{F}\le\sum_{i=1}^{N}\sqrt{\text{Tr}\left(\Context i{\tau}\Context i{\tau}^{T}\Context i{\tau}\Context i{\tau}^{T}\right)}\le N.
\]

By Assumptions 3 and 4, $\left\{ \sum_{i=1}^{N}\Context i{\tau}\Context i{\tau}^{T}\right\} _{\tau=1}^{t}$
are independent random variables such that $\Expectation\left[\sum_{i=1}^{N}\Context i{\tau}\Context i{\tau}^{T}\right]\ge N\phi^{2}>0$.
Let $\delta\in(0,1)$ be given. 
By Lemma \ref{lem:min_eigenvalue_concentration},
if we set $\lambda_{t}=4\sqrt{2}N\sqrt{t\log\frac{12t^2}{\delta}}$, 
\[
\left\{ \Mineigen{V_{t}}\right\} ^{-1}<\frac{1}{\phi^{2}Nt},
\]
holds with probability at least $1-\delta/(3t^2)$. 
Thus, the first term can be bounded by 
\begin{equation}
\left\{ \Mineigen{V_{t}}\right\} ^{-1}\lambda_{t}\le\frac{4\sqrt{\log\frac{12t^2}{\delta}}}{\sqrt{t}\phi^{2}}.
\label{eq:estimation_bound_1st_term}
\end{equation}

Now we will bound the second term in (\ref{eq:estimation_error_two_terms}).
Let $U_{i}(\tau):=\Context i{\tau}\Context i{\tau}^{T}(\Betatrunc{\tau}-\beta)$.
Then we can decompose $\DRError i\tau\Context i\tau$ as, 
\begin{equation}
\begin{split}
\DRError i{\tau}\Context i\tau= & U_{i}(\tau)+\frac{\Indicator{\Action{\tau}=i}}{\SelectionP i{\tau}}\left(\Reward i{\tau}-\Context i{\tau}^{T}\Betatrunc{\tau}\right)\Context i{\tau}\\
= & \left(1-\frac{\Indicator{\Action{\tau}=i}}{\SelectionP i{\tau}}\right)U_{i}(\tau)+\frac{\Indicator{\Action{\tau}=i}}{\SelectionP i{\tau}}\Error i{\tau}\Context i{\tau}\\
:= & D_{i}(\tau)+E_{i}(\tau).
\end{split}
\label{eq:DR_error_decomposition}
\end{equation}

Let $D_{\tau}:=\sum_{i=1}^{N} D_i(\tau)$. Since $U_i(\tau)$ is $\History{\tau}$-measurable, the conditional expectation of $D_{\tau}$ is
\begin{equation*}
\begin{split}
\CE{D_{\tau}}{\History{\tau}} &=  \CE{\sum_{i=1}^{N}D_{i}(\tau)}{\History{\tau}} 
=  \sum_{i=1}^{N}\CE{\left(1-\frac{\Indicator{\Action{\tau}=i}}{\SelectionP i{\tau}}\right)}{\History{\tau}}U_{i}(\tau) \\
&=  \sum_{i=1}^{N}\left(1-\frac{\SelectionP i{\tau}}{\SelectionP i{\tau}}\right)U_{i}(\tau)
=  0
\end{split}
\end{equation*}
Thus, $\left\{ \sum_{u=1}^{\tau} D_{\tau} \right\}_{\tau=1}^{t}$ is a martingale sequence on $\left(\Real^{d},\norm{\cdot}_{2}\right)$ with respect to $\History{\tau}$. 
By Lemma \ref{lem:dim_reduction}, since $\left(\Real^{d},\norm{\cdot}_{2}\right)$ is a Hilbert space,
there exists a martingale sequence $\left\{ P_{\tau}\right\} _{\tau=1}^{t}=\left\{ \left(P_{\tau}^{(1)},P_{\tau}^{(2)}\right)^{T}\right\} _{\tau=1}^{t}$
on $\Real^{2}$ such that 
\begin{equation}
\norm{\sum_{u=1}^{\tau}D_{u}}_{2}=\norm{P_{\tau}}_{2}, \quad \norm{D_{\tau}}_{2}=\norm{P_{\tau}-P_{\tau-1}}_{2}
\label{eq:hilbert_projection}
\end{equation}
and $P_0=0$, for any $\tau=1,\ldots,t$. 
Since $\norm{\Betatrunc{\tau}-\beta}_{2}\le b$, for $r=1,2$
\begin{equation*}
\begin{split}
\abs{P_{\tau}^{(r)}-P_{\tau-1}^{(r)}}\le\norm{P_{\tau}-P_{\tau-1}}_{2}=&\norm{D_{\tau}}_{2}\\=&\norm{\sum_{i=1}^{N}\left(1-\frac{\Indicator{\Action{\tau}=i}}{\SelectionP i{\tau}}\right)U_{i}(\tau)}\\
\le&\sum_{i=1}^{N}\abs{1-\frac{\Indicator{\Action{\tau}=i}}{\SelectionP i{\tau}}}\norm{U_{i}(\tau)}_{2}\\
\le&\sum_{i=1}^{N}\abs{1-\frac{\Indicator{\Action{\tau}=i}}{\SelectionP i{\tau}}}\norm{\Betatrunc{\tau}-\beta}_{2}\\
\le&\left(N-1+\frac{1}{\SelectionP{\Action{\tau}}{\tau}}-1\right)b\\
\le&\left(N+\SelectionP{\Action{\tau}}{\tau}^{-1}\right)b.
\end{split}
\end{equation*}
By Lemma \ref{lem:chung_lemma}, there exists a martingale sequence $\left\{ N_{\tau}^{(r)}\right\}_{\tau=1}^{t}$ such that $\abs{N_{\tau}^{(r)}-N_{\tau-1}^{(r)}}\le(N+\gamma^{-1})b$, for all $\tau=1,\ldots,t$ and 
\begin{equation}
\left\{ N_{t}^{(r)}\neq P_{t}^{(r)}\right\} \subset\bigcup_{\tau=1}^{t}\left\{ \abs{P_{\tau}^{(r)}-P_{\tau-1}^{(r)}}>(N+\gamma^{-1})b\right\} \subset\bigcup_{\tau=1}^{t}\left\{ \SelectionP{\Action{\tau}}{\tau}\le\gamma\right\}.
\label{eq:bad_set_inclusion}
\end{equation}

Thus, by (\ref{eq:hilbert_projection}) and (\ref{eq:bad_set_inclusion}), for any $x>0$,
\begin{equation*}
\begin{split}
\Probability\left(\norm{\sum_{u=1}^{t}D_{u}}_{2}>x,\;\bigcap_{\tau=1}^{T}\left\{ \SelectionP{\Action{\tau}}{\tau}>\gamma\right\} \right)
=&\Probability\left(\norm{P_{t}}_{2}\ge x,\;\bigcap_{\tau=1}^{T}\left\{ \SelectionP{\Action{\tau}}{\tau}>\gamma\right\} \right)\\
\le&\Probability\left(\sum_{r=1}^{2}\abs{P_{t}^{(r)}}\ge x,\;\bigcap_{\tau=1}^{t}\left\{ \SelectionP{\Action{\tau}}{\tau}>\gamma\right\} \right)\\
\le&\sum_{r=1}^{2}\Probability\left(\abs{P_{t}^{(r)}}\ge\frac{x}{2},\;\bigcap_{\tau=1}^{t}\left\{ \SelectionP{\Action{\tau}}{\tau}>\gamma\right\} \right)\\
\le&\sum_{r=1}^{2}\Probability\left(\abs{P_{t}^{(r)}}\ge\frac{x}{2},\;N_{t}^{(r)}=P_{t}^{(r)}\right)\\
\le&\sum_{r=1}^{2}\Probability\left(\abs{N_{t}^{(r)}}\ge\frac{x}{2}\right).
\end{split}
\end{equation*}
Since $N_{\tau}^{(r)}$ has bounded differences, we can apply Lemma \ref{lem:Azuma_Hoeffding_ineqaulity} to have
\begin{equation*}
\sum_{r=1}^{2} \Probability\left(\abs{N_{t}^{(r)}}\ge\frac{x}{2}\right) \le 4 \exp\left(-\frac{x^{2}}{8tb^{2}\left(N+\gamma^{-1}\right)^{2}}\right).
\end{equation*}

Thus, with probability at least $1-\delta/(3t^2)$, 
\begin{equation}
\norm{\sum_{\tau=1}^{t}D_{\tau}}_{2}\le2\sqrt{2}b(N+\gamma^{-1})\sqrt{\log \frac{12t^2}{\delta}}
\label{eq:estimation_bound_D}
\end{equation}
holds with the event $\bigcap_{t=1}^{T} \{\SelectionP{\Action{t}}{t}>\gamma\}$.

Now we will bound the $E_{i}(\tau)$ term in (\ref{eq:DR_error_decomposition}).
Under the event $\bigcap_{t=1}^{T} \{\SelectionP{\Action{t}}{t}>\gamma\}$, we have
\begin{equation*}
\sum_{\tau=1}^{t}\sum_{i=1}^{N}E_{i}(\tau)
=\sum_{\tau=1}^{t}\frac{\Error{\Action{\tau}}{\tau}}{\SelectionP{\Action{\tau}}{\tau}}\Context{\Action{\tau}}{\tau}
=\sum_{\tau=1}^{t}\frac{\Indicator{\SelectionP{\Action{t}}{t}>\gamma}\Error{\Action{\tau}}{\tau}}{\SelectionP{\Action{\tau}}{\tau}}\Context{\Action{\tau}}{\tau}
\end{equation*}
For each $\tau\ge1$, define a filtration $\Filtration{\tau-1}:=\History{\tau}\cup \{\Action{\tau}\}$.
Then $\Context{\Action{\tau}}{\tau}$ is $\Filtration{\tau-1}$-measurable.
By Assumption 2, for any $\lambda\in\Real$,
\begin{equation*}
\CE{\exp\left(\lambda\frac{\Indicator{\SelectionP{\Action{t}}{t}>\gamma}\Error{\Action{\tau}}{\tau}}{\SelectionP{\Action{\tau}}{\tau}}\right)}{\Filtration{\tau-1}}
\le \exp\left(\frac{\lambda^{2}\Indicator{\SelectionP{\Action{t}}{t}>\gamma}\sigma^2}{2\SelectionP{\Action{\tau}}{\tau}^{2}}\right)
\le\exp\left(\frac{\lambda^{2}\sigma^2}{2\gamma^{2}}\right),
\end{equation*}
almost surely.
Since $\norm{\Context{\Action{\tau}}{\tau}}_{2}\le 1$, by Lemma \ref{lem:eta_x_bound}, there exists an absolute constant $C>0$ such that, with probability at least $1-\delta/(3t^2)$,
\begin{equation}
\norm{\sum_{\tau=1}^{t}\sum_{i=1}^{N}E_{i}(\tau)}_{2}\le2 C \sigma \gamma^{-1}\sqrt{t}\sqrt{\log\frac{12t^2}{\delta}}.
\label{eq:estimation_bound_E}
\end{equation}

Thus, with (\ref{eq:estimation_bound_1st_term}), (\ref{eq:estimation_bound_D}), and (\ref{eq:estimation_bound_E}), under the event $\bigcap_{t=1}^{T}\left\{ \SelectionP{\Action{t}}{t}>\gamma  \right\}$, we have
\begin{equation}
\begin{split}
\norm{\Estimator t-\beta}_{2}\le & \frac{4\sqrt{\log\frac{12t^2}{\delta}}}{\sqrt{t}\phi^{2}}+\frac{1}{\phi^{2}Nt}\left(4\left(N+\gamma^{-1}\right)b\sqrt{t}\sqrt{\log\frac{12t^2}{\delta}}+2C\sigma\gamma^{-1}\sqrt{t}\sqrt{\log\frac{12t^2}{\delta}}\right)\\
\le & \frac{4+4b+\gamma^{-1}N^{-1}\left(4b+2C\sigma\right)}{\phi^{2}\sqrt{t}}\sqrt{\log\frac{12t^2}{\delta}} \\
\le & \frac{4+4b+2\left(4b+2C\sigma\right)}{\phi^{2}\sqrt{t}}\sqrt{\log\frac{12t^2}{\delta}} \\
:= & \frac{C_{b,\sigma}}{\phi^{2}\sqrt{t}}\sqrt{\log\frac{12t^2}{\delta}},
\end{split}
\label{eq:estimation_t_bound}
\end{equation}
with probability at least $1-\delta/t^2$.
Since (\ref{eq:estimation_t_bound}) holds for all $t=1,\ldots,T$, 
\begin{equation*}
\begin{split}
&\le \Probability\left(\bigcup_{t=1}^{T}\left\{\norm{\Estimator{t}-\beta}_{2} > \frac{C_{b,\sigma}}{\phi^{2}\sqrt{t}}\sqrt{\log\frac{12t^2}{\delta}} \right\}, \; \bigcap_{t=1}^{T}\left\{ \SelectionP{\Action{t}}{t} > \gamma\right\} \right)\\ 
&\le \Probability\left(\bigcup_{t=1}^{T}\left\{\norm{\Estimator{t}-\beta}_{2} > \frac{C_{b,\sigma}}{\phi^{2}\sqrt{t}}\sqrt{\log\frac{12t^2}{\delta}} \right\}, \; \bigcap_{t=1}^{T}\left\{ \SelectionP{\Action{t}}{t} > \gamma\right\} \right) \\
&\le \sum_{t=1}^{T} \Probability\left(\norm{\Estimator{t}-\beta}_{2} > \frac{C_{b,\sigma}}{\phi^{2}\sqrt{t}}\sqrt{\log\frac{12t^2}{\delta}}, \; \bigcap_{t=1}^{T}\left\{ \SelectionP{\Action{t}}{t} > \gamma\right\} \right) \\
&\le \delta.
\end{split}
\end{equation*}
\end{proof}

\section{Proof of Lemma \ref{lem:eta_x_bound}}
\begin{proof}
Fix a $t\ge1$. 
Since for each $\tau=1,\ldots,t$, $\CE{\eta(\tau)}{\Filtration{\tau-1}}=0$ and $X(\tau)$ is $\Filtration{\tau-1}$-measurable, the stochastic process,
\begin{equation}
\left\{ \sum_{\tau=1}^{u}\eta(\tau)X(\tau)\right\} _{u=1}^{t}
\end{equation}
is a $\Real^{d}$-martingale. 
Since $(\Real^{d},\norm{\cdot}_{2})$ is a Hilbert space, by Lemma \ref{lem:dim_reduction}, there exists
a $\Real^{2}$-martingale $\{M_{u}\}_{u=1}^{t}$ such that 
\begin{equation}
\norm{\sum_{\tau=1}^{u}\eta(\tau)X(\tau)}_{2}=\norm{M_{u}}_{2},\;\norm{\eta(u)X(u)}_{2}=\norm{M_{u}-M_{u-1}}_{2},
\end{equation}
and $M_{0}=0$.
Set $M_{u}=(M_{1}(u),M_{2}(u))^{T}$. Then for each $i=1,2$, and $u\ge2$, by the assumption $\|X(u)\|_2\le 1$,
\begin{align*}
\abs{M_{i}(u)-M_{i}(u-1)}\le & \norm{M_{u}-M_{u-1}}_{2}\\
= & \norm{\eta(u)X(u)}_{2}\\
\le & \abs{\eta(u)}.
\end{align*}
By Lemma  \ref{lem:subgaussian}, $M_{i}(u)-M_{i}(u-1)$ is $C\sigma$-sub-Gaussian for some constant $C>0$. 
By Lemma \ref{lem:Azuma_Hoeffding_ineqaulity},  for $x>0$,
\begin{align*}
\Probability\left(\abs{M_{i}(t)}>x\right)= & \Probability\left(\abs{\sum_{u=1}^{t}M_{i}(u)-M_{i}(u-1)}>x\right)\\
\le & 2\exp\left(-\frac{x^{2}}{2tC^2\sigma^2}\right),
\end{align*}
for each $i=1,2$.
Thus, with probability $1-\delta/(2t^2)$,
\[
M_{i}(t)^{2}\le2tC^2\sigma^2\log\frac{4t^2}{\delta}.
\]
In summary, with probability at least $1-\delta/t^2$,
\[
\norm{\sum_{\tau=1}^{t}\eta(\tau)X(\tau)}_{2}=\sqrt{M_{1}(t)^{2}+M_{2}(t)^{2}}\le2C\sigma\sqrt{t}\sqrt{\log\frac{4t^2}{\delta}}.
\]
\end{proof}

\section{Proof of Lemma \ref{lem:min_eigenvalue_concentration}}
\begin{proof}
For each $\tau=1,\ldots,t$, let $\Sigma_{\tau}=\CE{P(\tau)}{\Filtration{\tau-1}}$.
Since $P(\tau)$ and $\Sigma_{\tau}$ are symmetric matrices,
\[
\begin{split}
\Mineigen{\sum_{\tau=1}^{t}P(\tau)+\lambda_{t}I}=&\Mineigen{\sum_{\tau=1}^{t}P(\tau)}+\lambda_{t}\\
=&\Mineigen{\sum_{\tau=1}^{t}\left\{ P(\tau)-\Sigma_{\tau}\right\} +\sum_{\tau=1}^{t}\Sigma_{\tau}}+\lambda_{t}\\
\ge&\Mineigen{\sum_{\tau=1}^{t}\left\{ P(\tau)-\Sigma_{\tau}\right\} }+\sum_{\tau=1}^{t}\Mineigen{\Sigma_{\tau}}+\lambda_{t}\\
\ge&\Mineigen{\sum_{\tau=1}^{t}\left\{ P(\tau)-\Sigma_{\tau}\right\} }+\phi^{2}t+\lambda_{t}.
\end{split}
\]
The last inequality uses the fact that $\Mineigen{\Sigma_{\tau}}\ge\phi^2$ for all $\tau$.
\begin{equation}
\begin{split}
\Probability\left(\Mineigen{\sum_{\tau=1}^{t}P(\tau)+\lambda_{t}I}\le\phi^{2}t\right)\le&\Probability\left(\Mineigen{\sum_{\tau=1}^{t}\left\{ P(\tau)-\Sigma_{\tau}\right\} }+\lambda_{t}\le0\right)\\
=&\Probability\left(\lambda_{\max}\left(\sum_{\tau=1}^{t}\left\{ \Sigma_{\tau}-P(\tau)\right\} \right)\ge\lambda_{t}\right)\\
\le&\Probability\left(\norm{\sum_{\tau=1}^{t}\left\{ \Sigma_{\tau}-P(\tau)\right\} }_{F}\ge\lambda_{t}\right).
\label{eq:min_eigen_frobenious_norm}
\end{split}
\end{equation}
Set $S_{u}=\sum_{\tau=1}^{u}\left\{\Sigma_{\tau}-P(\tau)\right\}$.
Then $\{S_{u}\}_{u=1}^{t}$ can be regarded as a martingale sequence on $\Real^{d\times d}$ with respect to $\left\{ P(\tau)\right\} _{\tau=1}^{t}$.
Note that $\left(\Real^{d\times d},\norm{\cdot}_{F}\right)$ is a Hilbert space. 
By Lemma \ref{lem:dim_reduction}, there exists a martingale sequence $\left\{ D_{u}=(D_{1}(u),D_{2}(u))^{T}\right\} _{u=1}^{t}$ on $\Real^{2}$ such that 
\begin{equation}
\norm{S_{u}}_{F}=\sqrt{D_{1}(u)^{2}+D_{2}(u)^{2}},\quad
\norm{P(u)-\Sigma_{u}}_{F}=\norm{D_{u}-D_{u-1}}_{2},
\label{eq:projected_matrix_martingale}
\end{equation}
for any $u\ge1$, and $D_{0}=0$.
Then, for any $i=1,2$, 
\[
\abs{D_{i}(u)-D_{i}(u-1)}^{2}\le\norm{D_{u}-D_{u-1}}_{2}^{2}=\norm{P(u)-\Sigma_{u}}_{F}^{2}
\]
Since $\norm{P(u)-\Sigma_{u}}_{F}\le2c$, we can apply Lemma \ref{lem:Azuma_Hoeffding_ineqaulity}
for $D_{1}(\tau)$, and $D_{2}(\tau)$, respectively. For any $i=1,2$,
and for any $x>0$, 
\[
\Probability\left(\abs{D_{i}(t)}\ge x\right)\le2\exp\left(-\frac{x^{2}}{8c^{2}t}\right).
\]
From (\ref{eq:min_eigen_frobenious_norm}) and \eqref{eq:projected_matrix_martingale},
\begin{equation*}
\begin{split}
\Probability\left(\Mineigen{\sum_{\tau=1}^{t}P(\tau)+\lambda_{t}I}\le\phi^{2}t\right)
&\le\Probability\left(\norm{S_{t}}_{F}\ge\lambda_{t}\right) \\
&=\Probability\left(\sqrt{D_{1}(t)^{2}+D_{2}(t)^{2}} \ge \lambda_{t}\right) \\
&\le\Probability\left(\abs{D_{1}(t)}+\abs{D_{2}(t)}\ge\lambda_{t}\right) \\
&\le \Probability\left(\abs{D_{1}(t)}\ge\frac{\lambda_{t}}{2}\right) +\Probability\left(\abs{D_{2}(t)}\ge\frac{\lambda_{t}}{2}\right) \\
&\le 4\exp\left(-\frac{\lambda_{t}^{2}}{32c^{2}t}\right).
\end{split}
\end{equation*}
Thus, for any $\delta\in(0,1)$, if $\lambda_{t}\ge4\sqrt{2}c\sqrt{t}\sqrt{\log\frac{4t^2}{\delta}}$,
then with probability at least $1-\delta/t^2,$ 
\[
\Mineigen{\sum_{\tau=1}^{t}P(\tau)+\lambda_{t}I}>\phi^{2}t.
\]
\end{proof}

\section{Implementation details}
\subsection{Efficient calculation of the sampling probability}
In our proposed algorithm, we use quasi-Monte Carlo estimation to calculate the sampling probability, $\Pisampled{i}{t}$.
At round $t$, for each \(i=1,\ldots,N\), define \(Z_{i}=\frac{\Context {i}{t}^T\left(\BetaSampled{i}{t}-\Estimator{t-1}\right)}{v\norm{\Context{i}{t}}_{V_t^{-1}}}\).
Then, \(Z_1,\ldots,Z_N\) are IID standard Gaussian random variables.
For each \(i=1,\ldots,N\),
\begin{align*}
\Pisampled{i}{t}= & \CP{\Context{i}{t}^T\BetaSampled{i}{t} \ge \Context{j}{t}^T\BetaSampled{j}{t},\forall j \neq i}{\History t} \\
= & \CP{\frac{\norm{\Context{i}{t}}_{V_t^{-1}}}{\norm{\Context{j}{t}}_{V_t^{-1}}} Z_{i}\ge Z_{j}+\frac{\left(\Context jt-\Context it\right)^T\Estimator{t-1}}{v\norm{\Context{j}{t}}_{V_t^{-1}}},\forall j\neq i}{\History t}
\end{align*}
let \(f\) and \(F\) be the density and the distribution function of the standard Gaussian random variables, respectively.
Since \(Z_i\), and \(\left\{ Z_j \right\}_{j\neq i}\) are independent, the selection probability can be written as,
\[
\Pisampled{i}{t} = \int \prod_{j\neq i} F\left( \frac{\norm{\Context{i}{t}}_{V_t^{-1}}}{\norm{\Context{j}{t}}_{V_t^{-1}}} z +\frac{\left(\Context it-\Context jt\right)^T\Estimator{t-1}}{v\norm{\Context{j}{t}}_{V_t^{-1}}} \right) f(z)dz.
\]
This can be estimated by,
\begin{equation}
\label{eq:pi_estimate}
\frac{1}{M}\sum_{m=1}^{M}F\prod_{j\neq i} \left( \frac{\norm{\Context{i}{t}}_{V_t^{-1}}}{\norm{\Context{j}{t}}_{V_t^{-1}}} Z^{(m)} +\frac{\left(\Context it-\Context jt\right)^T\Estimator{t-1}}{v\norm{\Context{j}{t}}_{V_t^{-1}}} \right),
\end{equation}
where \(Z^{(m)}\) is the standard Gaussian random variables.

In this way, we can compute $\Pisampled{i}{t}$ without sampling \(\BetaSampled{i}{t}\) \(M\times N\) times from \(N(\Estimator{t-1},v_{t}I)\).
The error of the quasi Monte Carlo method is bounded by \(O\left(\frac{(\log M)^s}{M}\right)\), where \(s\) is the dimension of the domain of function to integrate.
If we sample \(\BetaSampled{i}{t}\) \(M \times N\) times, it gives \(O\left(\frac{(\log M)^{N-1}}{M}\right)\) error.
In contrast, using (\ref{eq:pi_estimate}) reduces the error to \(O\left(\frac{\log M}{M}\right)\).

In our simulation studies, we use \texttt{sobol}\_\texttt{seq} module in Python 3 to generate the quasi-Monte Carlo samples.
The number of samples is \(M=200\) in \texttt{BLTS} and \texttt{DRTS}.
We plot the estimator of \(\Pisampled{i}{t}\) using \(m=1, \ldots, 200\) quasi-Monte Carlo samples, and observe that it converges within the small errors.


\begin{figure}[t]
\centering
\includegraphics[width=\linewidth]{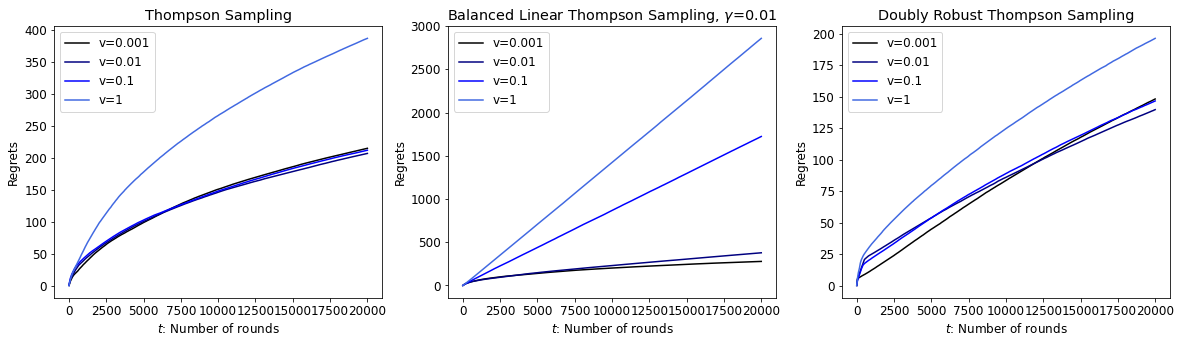}
\caption{
A comparison of the cumulative regrets of \texttt{LinTS} (left), \texttt{BLTS} (middle), and \texttt{DRTS} (right) on various \(v\) when $d=20$ and $N=20$.
Each line shows the averaged cumulative regrets over 10 repeated experiments.
}
\label{fig:hyperparameters}
\end{figure}

\subsection{Simulation results with various hyperparameters}

In this subsection, we report the performance of the three algorithms, (i) \texttt{LinTS}, (ii) \texttt{BLTS}, and (iii) the proposed \texttt{DRTS}, with various hyperparameters.
As described in Section \ref{sec:simulation}, the hyperparameter sets are $v\in\{0.001,0.01,0.1,1\}$ for all three algorithms and $\gamma\in\{0.01, 0.5, 0.1\}$ for \texttt{BLTS}.

Figure \ref{fig:hyperparameters} shows the comparison of the three algorithms on various hyperparameter \(v\), when \(d=20, N=20, \gamma=0.01\).
We find that the performance of the three algorithms do not change much when $v\leq 0.01$.
This trend is similar on different \(\gamma, N\), and \(d\).

\section{A review of approaches to missing data and doubly-robust method}
In this section, we review approaches to missing data and the doubly-robust method used in our proposed method.
First, we provide the approaches from a purely missing data point of view and how the doubly-robust method is motivated.   
In the second section, we show the procedures applying the doubly-robust method in bandit settings.

\subsection{Doubly-robust method in missing data}
There are two main approaches to missing data: imputation and inverse probability weighting (IPW).
Imputation is to fill in the predicted value of missing data from a specified model, and IPW is to use the observed records only but weight them by the inverse of the observation probability.  The doubly-robust method can be viewed as a combination of the two.

For illustrative purposes, consider the problem of estimating the marginal mean of $Y\in \mathbb{R}$, $\Expectation(Y)=:\mu$.   
Denoting $(Y_i-\mu)$ by $U_i(\mu)$, when all data are observed, $$U(\mu)=\sum_{i=1}^n U_i(\mu)=\sum_{i=1}^n (Y_i-\mu)=0,$$ gives an unbiased estimator of $\mu$, $\sum_{i=1}^n Y_i/n$,  and $U(\mu)$ is called an unbiased estimating function since $\Expectation[U(\mu)]=0$.   
Let $\delta_i$ be the observation indicator which takes value 1 if $Y_i$ is observed, 0, otherwise.  Suppose there are auxiliary variables,   $X_i\in \mathbb{R}^d$, and $X_i$'s are  observed for all $i$.  
Also denote the probability of observation by $P(\delta_i=1|X_i)=:\pi_i$.  
We assume $P(\delta_i=1|Y_i, X_i)=P(\delta_i=1|X_i)$, that is, the observation indicator is independent of $Y_i$.  This is called {\it missing at random} mechanism.  This assumption is required for the doubly robust method to be valid.  
Using the observed values only, the estimating equation for the observed data $$U_o(\mu)=\sum_{i=1}^n\delta_i U_i(\mu)=\sum_{i=1}^n \delta_i(Y_i-\mu)=0,$$ gives $\frac{\sum_{i=1}^n \delta_iY_i}{\sum_{i=1}^n \delta_i}$ as an estimator for $\mu$.  
This estimator may be biased since  $\Expectation U_o(\mu)\neq 0.$

The two main approaches modify the observed estimating function employing two new quantities,  $\Expectation(Y_i|X_i)$ and  $\pi_i$.  
These two quantities are usually unknown and we need to specify models.  
Therefore the two approaches  require assumptions for auxiliary models: the imputation model, $\Expectation(Y_i|X_i;\beta)$, and the model for observation probability, $\pi_i(\phi)$.  
The validity of each approach depends on the correct specification of the auxiliary model assumptions. 
The qualifier `auxiliary' comes from the fact that  these models are not needed when there is no missing data.  
In IPW, one constructs an unbiased estimating equation by amplifying the observed record according to the inverse of the observation probability as follows: $$\sum_{i=1}^n\frac{\delta_i}{\pi_i(\phi)}U_i(\mu)=\sum_{i=1}^n \frac{\delta_i}{\pi_i(\phi)}(Y_i-\mu).$$  
If $\pi(\phi)$ is correctly specified, i.e., $\pi=\pi(\phi)$,  $\Expectation(\sum_{i=1}^n\frac{\delta_i}{\pi_i(\phi)}U_i(\mu))=0$, hence the resulting IPW estimator is valid.  
In the imputation method, we replace missing $Y_i$ with $\Expectation(Y_i|X_i;\beta)$ and the estimator is the solution of $U^{IMP}(\mu,\beta)=0$ where
\begin{equation*}
\begin{split}
U^{IMP}(\mu,\beta)&=\sum_{i=1}^n \left[\delta_i U_i(\mu)+(1-\delta_i)\Expectation(U_i(\mu)|X_i;\beta)\right]\\
&=\sum_{i=1}^n \left[\Expectation(Y_i|X_i;\beta) +\delta_i\{Y_i-\Expectation(Y_i|X_i;\beta)\}-\mu \right].
\end{split}
\end{equation*}

The doubly robust (DR) method \citep{robins1994, bang2005doubly} was initially motivated by attempting to improve the efficiency of the IPW method.   
Note that we can construct an auxiliary unbiased estimating function $(\frac{\delta_i}{\pi_i(\phi)}-1)$.  
Geometrically we can reduce the norm of the estimating function $\frac{\delta_i}{\pi_i(\phi)}U_i(\mu)$ by subtracting the projection on to the nuisance tangent space formed from $(\frac{\delta_i}{\pi_i(\phi)}-1)$. 
The nuisance tangent space is the closed linear span of $B(\frac{\delta_i}{\pi_i(\phi)}-1)$ for some $B\in\mathbb{R}^{d}$, and the projection onto the nuisance tangent space is $$\sum_{i=1}^{n}\frac{\delta_i-\pi_i(\phi)}{\pi_i(\phi)}\Expectation(U_i|X_i;\beta).$$
After subtraction, the DR estimating function has a form
\begin{equation*}
\begin{split}
U^{DR}(\mu,\beta,\phi) &= \sum_{i=1}^n  \left[ \frac{\delta_i}{\pi_i(\phi)}U_i(\mu)+
 (1- \frac{\delta_i}{\pi_i(\phi)}) \Expectation(U_i| X_i;\beta)  \right]\\
& = \sum_{i=1}^n \left[ \Expectation(U_i| X_i;\beta)+\frac{\delta_i}{\pi_i(\phi)} \{U_i(\mu)-\Expectation(U_i(\mu)|X_i;\beta)\}\right] .
\end{split}
\end{equation*}
Note that when you replace $\delta_i$ in $U^{IMP}(\mu)$ with $\frac{\delta_i}{\pi_i(\phi)}$, you obtain $U^{DR}(\mu)$.
The DR method  requires both auxiliary models.  
However,  its validity is guaranteed when {\it either} of the models is correct.   
To verify, if the imputation model is correctly specified, i.e., $\Expectation[U_i(\mu)-\Expectation(U_i(\mu)|X_i;\beta)|X_i]=0$, we have
$$\Expectation\{U^{DR}(\mu,\beta,\phi)\}=\Expectation\sum_{i=1}^n \left[ \Expectation(U_i| X_i)-\frac{\delta_i}{\pi_i(\phi)} \{U_i(\mu)-\Expectation(U_i(\mu)|X_i)\}\right] =\sum_{i=1}^n \Expectation\Expectation(U_i| X_i)=0$$ even if the $\pi$ model is misspecified, i.e., $\pi_i(\phi) \neq \pi_i$.  
If the observation model is correctly specified, $\pi_i(\phi) = \pi_i$, then  $\Expectation(1-\frac{\delta_i}{\pi_i}|X_i)=0$, and
 $$\Expectation\{U^{DR}(\mu,\beta,\phi)\}=\sum_{i=1}^n \Expectation\left[ \frac{\delta_i}{\pi_i}U_i(\mu)+
\left\{ (1-\frac{\delta_i}{\pi_i}) \Expectation(U_i| X_i;\beta) \right\} \right]= \sum_{i=1}^n  \Expectation\left[ \frac{\delta_i}{\pi_i}U_i(\mu)\right]=0,$$ even if the imputation model is misspecified, i.e., 
$\Expectation[U_i(\mu)|X_i]\neq\Expectation[U_i(\mu)|X_i;\beta)]$.  
Therefore when   {\it either} of the models is correct, $U^{DR}(\mu)$ is unbiased and with other technical conditions, the estimator can be shown to be consistent.  
That is why the qualifier {\it doubly robust} is adopted.  
The construction of the DR estimating function is possible because we have two unbiased estimating functions.

\subsection{Application to bandit settings}
In bandit settings, the missingness is controlled since the learner selects the arm.  
Therefore, the probability of observation or selection is known and  the DR estimator is guaranteed to be valid although the imputation model for missing reward is incorrectly specified.   
The merit of the DR estimator in the bandit setting is that we can utilize the observed contexts from selected or unselected arms.   
Below we describe the DR method in the contextual bandit setting.

Let $\pi_{i}(t):=\mathbb{P}(\Action{t}=i| \mathcal{H}_t)$ be the probability of selecting arm $i$ at round $t$. 
As defined in the manuscript, the DR pseudo-reward is 
\begin{equation}
\DRreward it= \left\{ 1-\frac{\Indicator{i=\Action{t}}}{\SelectionP it}\right\} \Context it^{T}\Betatrunc{t}
+\frac{\Indicator{i=\Action{t}}}{\SelectionP it}\Reward{\Action{t}}t,
\label{eq:DRreward_supp}
\end{equation}
for some $\Betatrunc{t}$ depending on $\History t$. 
The pseudo-reward \eqref{eq:DRreward_supp} comes from the following procedures. 
First we construct an unbiased estimating function also known as the IPW score, 
\begin{equation}
\sum_{\tau=1}^{t}\sum_{i=1}^{N}\frac{\Indicator{i=\Action{\tau}}}{\SelectionP i{\tau}}\Context i{\tau}\left(\Reward i{\tau}-\Context i{\tau}^{T}\beta\right),
\label{eq:IPW_score}
\end{equation}
where only the pairs $(\Context it, \Reward it)$ from the selected arms are contributed according the weight of the inverse of $\SelectionP it$.
Setting this score equal to $0$ and solving $\beta$ gives the estimator used in \citet{dimakopoulou2019balanced}.
Now we can subtract the projection on the nuisance tangent space from (\ref{eq:IPW_score}).  
The nuisance tangent space is the closed linear span of $B(\frac{\mathbb{I}(i=a(t)}{\pi_i(t)}-1)$ for some $B\in\mathbb{R}^{d}$, and the projection onto the nuisance tangent space is
\begin{equation*}
\sum_{\tau=1}^{t}\sum_{i=1}^{N}\frac{\Indicator{i=\Action{\tau}}-\SelectionP i{\tau}}{\SelectionP i{\tau}}\Context i{\tau}\left(E(Y_{i}(\tau)|\History{\tau})-\Context i{\tau}^{T}\beta\right).
\end{equation*}
When the projection is subtracted from the (\ref{eq:IPW_score}) after replacing $E(Y_i(t)|\History t)$ with $\Context it^{T}\Betatrunc{t}$, the IPW score becomes the efficient score, 

\begin{equation}
\sum_{\tau=1}^{t}\sum_{i=1}^{N}\Context i{\tau}\left(\DRreward i{\tau}-\Context i{\tau}^{T}\beta\right).
\label{eq:efficient_score}
\end{equation}
Any $\Betatrunc{t}$ that depends on $\History t$ serves the purpose of imputation.
Due to the doubly robustness property, $\Context it^{T}\Betatrunc{t}$ does not have to be an unbiased estimator of  $E(Y_i(t)|\History t)$.
We recommend setting $\Betatrunc{t}$ as the ridge regression estimator based on the selected arms only.
The expression
(\ref{eq:efficient_score}) 
resembles the score when the rewards for all arms were observed, if $\Reward it$  is replaced with $\DRreward it$.

Our proposed estimator $\Estimator t$ is a solution of (\ref{eq:efficient_score}) with
a regularization parameter $\lambda_{t}$:
\begin{equation*}
\Estimator t = \left(\sum_{\tau=1}^{t}\sum_{i=1}^{N}\Context i{\tau}\Context i{\tau}^{T}+\lambda_{t}I\right)^{-1}\left(\sum_{\tau=1}^{t}\sum_{i=1}^{N}\Context i{\tau}\DRreward i{\tau}\right).
\end{equation*}
Harnessing the pseudo-rewards defined in (\ref{eq:DRreward_supp}), we can make use of all contexts rather than just selected contexts. 
The use of all contexts instead of $\Context{\Action{t}}t$ induces the improvement in the regret bound of the proposed algorithm.
\citet{kim2019doubly} also suggests DR estimator, but it uses Lasso estimator from the following pseudo-reward
$$\DRreward it=\bar{X}(t)^T\hat{\beta}(t-1)+\frac{1}{N}\frac{Y_{a(t)}(t)-b_{a(t)}(t)^T\hat{\beta}(t-1)}{\pi_{a(t)}(t)},$$
where $\bar{X}(t)=\frac{1}{N}\sum_{i=1}^N X_i(t)$.  
This estimator is of an aggregated form.  
As described in the text,  the estimator using the aggregated pseudo-reward does not permit the regret decomposition as equation \eqref{eq:decomposition} in the paper.

\section{Limitations of our work}
\begin{enumerate}
    \item The regret bound is constructed under the assumption that the contexts are independent over rounds (Assumption 3) and the covariance matrix is positive definite (Assumption 4).
    When using our proposed algorithm, one should check that the contexts satisfies the two assumptions.
    When the contexts violates the two assumptions, the improved regret bound \eqref{eq:regret_bound} might not hold.
    \item Our proposed algorithm, \texttt{DRTS} requires additional computations for $\Pisampled{i}{t}$, $\SelectionP{i}{t}$ and the imputation estimator $\check{\beta}$.
    To lessen this computational burdens we developed an efficient way to compute $\Pisampled{i}{t}$ and $\SelectionP{i}{t}$.
\end{enumerate}

\section{Computation of the selection probability}
\label{sec:pi_computation}
In this section, we provide details of how to compute the selection probability, $\SelectionP{i}{t}$.
Because the counter example in Remark~\ref{rem:chung_lemma} can be applied our case when $N=2$ and $\Pisampled{1}{t}=1-\delta/t^2$ and $\Pisampled{2}{t}=\delta/t^2$, for some $t\in[T]$.
To avoid this example, we adjust our $\Pisampled{i}{t}$ to $\SelectionP{i}{t}$ such that
\[
\pi_{i}(t):=
\begin{cases}
\gamma+\epsilon_{t} & \Pisampled{i}{t}>\gamma\\
\frac{\gamma}{2} & \Pisampled{i}{t}\le\gamma
\end{cases}
\]
where $\epsilon_{t}>0$ is a constant that makes $\sum_{i=1}^{N}\SelectionP{i}{t}=1$ hold for each $t\in[T]$.
In this way, we obtain 
\[
\min_{i\in[N]}\SelectionP{i}{t} \ge \frac{\gamma}{2} \ge \frac{1}{2(N+1)} > \frac{\delta}{t^2}, 
\]
and this avoids the counter example in Remark~\ref{rem:chung_lemma}.
By resampling $\SelectionP{i}{t}$ instead of $\Pisampled{i}{t}$, we obtain the same theoretical results and the regret bound in Theorem~\ref{thm:regret_bound} holds accordingly.
\end{document}